\documentclass[final]{colt2022} 
\usepackage{comment}

\usepackage{usual_macros}
\usepackage{usual_packages}

\title[Label noise (stochastic) gradient descent implicitly solves the Lasso]{Label noise (stochastic) gradient descent implicitly solves the Lasso for quadratic parametrisation}
\usepackage{times}
 
\coltauthor{%
 \Name{Loucas Pillaud-Vivien} \Email{loucas.pillaud-vivien@epfl.ch}  \\
 \addr EPFL
 \AND
 \Name{Julien Reygner} \Email{julien.reygner@enpc.fr}\\
 \addr \'Ecole des Ponts%
 \AND \Name{Nicolas Flammarion} \Email{nicolas.flammarion@epfl.ch}\\
 \addr EPFL
}

\usepackage{upgreek}

\newcommand{\Interp}{\mathcal{I}_+}
\newcommand{\supp}{\mathrm{supp}}
\newcommand{\Sstar}{S^\*}

\newcommand{\Mx}{\mathsf{x}} 
\newcommand{\My}{\mathsf{y}} 
\newcommand{\MX}{\mathsf{X}} 
\newcommand{\Mh}{\mathsf{h}} 
\newcommand{\ML}{\mathsf{L}} 
\newcommand{\Mdelta}{\updelta} 

\newcommand{\Rx}{x} 
\newcommand{\Ry}{y} 
\newcommand{\RX}{X} 
\newcommand{\Rh}{h} 
\newcommand{\RL}{L} 
\newcommand{\Rdelta}{\delta} 
\newcommand{\Rz}{z}
\newcommand{\RZ}{Z}

\newcommand{\RRd}{R^{\Rdelta}} 
\newcommand{\Mud}{\mu^{\Rdelta}} 
\newcommand{\Uz}{u_*}
\newcommand{\RZd}{\zeta^{\Rdelta}}
\renewcommand{\vv}{\tilde{v}}
\renewcommand{\ww}{\tilde{w}}
\newcommand{\RZdp}{\tilde{\zeta}^{\Rdelta}}
\newcommand{\Xid}{\xi^{\Rdelta}}
\newcommand{\RRdb}{\bar{R}^{\Rdelta}}
\newcommand{\Zd}{Z^{\Rdelta}}

\begin{document}

\maketitle


\begin{abstract}%
Understanding the implicit bias of training algorithms is of crucial importance in
order to explain the success of overparametrised neural networks. In this paper, we
study the role of the label noise in the training dynamics of a quadratically parametrised model
through its continuous time version. We explicitly
characterise the solution chosen by the stochastic flow and prove that it implicitly solves a Lasso program. To fully
complete our analysis, we provide nonasymptotic convergence guarantees for the dynamics as well as conditions for support recovery. We
also give experimental results which support our theoretical claims. Our findings
highlight the fact that structured noise can induce better generalisation and
help explain the greater performances of stochastic dynamics as observed in practice.
\end{abstract}

\begin{keywords}%
  Label noise, Stochastic dynamics, Lasso, Sparse Regression.%
\end{keywords}

\section{Introduction}

The many successes of deep learning are undoubtedly equal to the theoretical mysteries that surround it. However, while theoretical explanations were quite weak a decade ago, some recent progresses have refined our understanding of neural networks: by proving convergence in some cases~\citep{mei2018mean,chizat2018global}, or clarifying the role of initialisation~\citep{jacot2018ntk,chizat2019lazy}. Still, one of their most surprising and unexplained aspect is their ability to generalise without explicit regularisation despite large overparametrisation~\citep{recht2017understanding}. 

In fact, due to the high expressivity of overparametrised neural networks, they carry a large freedom while fitting a data set, yet without any hurt in the generalisation performance. That is to say that the way they are trained (initialisation, algorithm, specific architecture) specifies a large part of their generalisation abilities.
This crucial aspect, often referred to as the \textit{implicit bias} or \textit{algorithmic regularisation}, has been a major line of research lately. For example, in the simple and prototypical least-square framework, it has been shown that both gradient descent and stochastic gradient descent converge towards the global solution which has the lowest squared distance from the initialisation~\citep{recht2017understanding}. For a linear parametrisation, \citet{soudry2018implicit} show in a seminal paper that gradient descent selects the max-margin classifier for logistic regression on separable data. Since then, many works have tried to characterise the implicit bias of specified settings: for classification with neural networks~\citep{Lyu2020Gradient,chizat2020implicit}, and regression for a large variety of nonconvex models~\citep{woodworth2020kernel,arora2019implicit} for which initialisation always plays a central role~\citep{maennel2018gradient}.

A common feature of these analyses is that they study (deterministic) gradient descents (GD), while it has been shown empirically that stochasticity may be of primary importance to match best generalisation guarantees~\citep{keskar2017large}. Hence, it is natural to try to understand the role of stochasticity induced by the mini-batch training procedure of stochastic gradient decent (SGD). It is often shown that SGD tends to move towards \textit{flat regions} of the training loss~\citep{pmlr-v97-zhu19e,chaudhari2018stochastic}. However,
the flat minima selection phenomenon does not appear very clearly and noise models taken to rigorously prove these are disputable. In this perspective, specific noise models to understand the role of stochasticity are primordial: an example of that is the fact that minibatch stochasticity of SGD is state dependent and cancels itself at global optima~\citep{wojtowytsch2021stochastic,pesme2021implicit,ali2020implicit}. Another important feature is that the noise has a specific geometry, e.g. in the least-square model it belongs to the span of the data inputs~\citep{recht2017understanding}. 

To understand this without suffering from the degeneracy at global optima one may resort to \textit{label noise} SGD, where  some noise is systematically added to the output at each step of the descent. This injected noise has been shown to be a good surrogate model that exemplifies the geometry and the state dependence of the noise carried by SGD before reaching zero training loss~\citep{haochen2021shape}. In this perspective, local implicit bias criteria have been sketched~\citep{damian2021label,blanc2020implicit} and notably, a limiting process has been introduced lately to formally explain how the label noise drives the dynamics~\citep{li2022happens}. However, all these correspond to local, nonexplicit and asymptotic results that might not be as satisfying as those proven in the deterministic case. We study such a label noise procedure in a nonconvex model, and provide an explicit, nonasymptotic description of the dynamics.  

The quadratic parametrisations which we consider have become popular lately~\citep{vavskevivcius2019implicit} since, despite their simplicity, they already enable to grasp the complexity of more general networks. Indeed, they highlight important aspects of the theoretical concerns of modern machine learning: the neural tangent kernel regime, the roles of overparametrisation and of the
initialisation~\citep{woodworth2020kernel}. In this literature, it has been shown that a $\ell_1$-sparcifying regularisation rules the implicit biasing, together with initialisation: indeed, when the initialisation goes infinitely small, GD~\citep{woodworth2020kernel} and SGD~\citep{pesme2021implicit} select a sparse interpolator of the data. However, an important drawback is that as initialisation gets very small, optimisation time gets very large. Hence the following question
\begin{center}
\textit{Does label noise help in recovering a sparse interpolator without infinitely small initialisation? }
\end{center}
To tackle this question, we study the label noise (stochastic) gradient descent through its continuous version, namely the
stochastic gradient flow (SGF). We stress that in our work, we attach
peculiar attention to the adequate modelling of the noise. Tools from Itô calculus are then leveraged
in order to derive exact formulas, quantitative bounds and interesting interpretations for our problem.

\subsection{Main contribution and paper organisation}

In Section~\ref{sec:setup}, we start by introducing the setup of our problem as well as the continuous stochastic model. Then, in Section~\ref{sec:results}, we state the main results on the dynamics convergence, deriving precise nonasymptotic statements, both in terms of time and noise. We informally formulate it here:
\begin{theorem*}[{\bfseries Informal}]
For any initialisation, the label noise stochastic gradient flow for quadratic parametrisation implicitly solves a weighted Lasso program. In consequence, under conditions on the design matrix, it recovers exactly the support of the ground-truth sparse estimator of the model.
\end{theorem*}
To reach this goal, we analyse thoroughly the dynamics, handling precisely its stochastic fluctuations. The proof sketch is depicted in Section~\ref{sec:dynamics}. We support our results experimentally and validate our model in Section~\ref{sec:experiments}.

\subsection{Additional related work}

A large part of the related work has already been covered in the introduction. Let us complete it here. Our study leverages continuous time stochastic differential equation modelling of discrete time dynamics. We refer to \citet{JMLR:v20:17-526} for a technical introduction to these techniques when related to machine learning problems. 

Let us also compare the present work with the recent literature on label noise driven GD. Two different points of view are taken in the literature. The aim of the first one pioneered by \cite{blanc2020implicit} is to show that such stochastic dynamics are biased towards optimising a hidden objective related to the curvature of the loss. However, it seems hard to conclude as their calculations are essentially both local and asymptotic. In the same spirit, one of the most conclusive works related to this approach is certainly the recent work of \citet{li2022happens} in which the authors exhibit a proper limiting dynamics upon the manifold of interpolators thanks to a time rescaling. Once again, the results shown are only asymptotic, nonquantitative and difficult to apprehend. In contrast, the aim of the present paper is to characterise quantitatively the convergence without resorting to any limiting argument. Finally, \citet{haochen2021shape} show a similar collapsing effect due to the label noise. However, their analysis relies on an extremely large noise (at least square of the dimension), so that our result on the large noise regime alone can be considered finer.

\subsection{Notations}

For $d \in \N^*$, $\R^d_+$ is the cone of vectors of $\R^d$ with nonnegative components. For vectors $\theta, \theta' \in \R^d$,  $\langle  \theta, \theta'\rangle$ denotes the standard scalar product of $\R^d$ and $\|\cdot\|_2$ its associated Euclidean norm. For a matrix $X \in \R^{n \times d}$, $\|X\|$ denotes the operator norm associated with $\|\cdot\|_2$. In case of a square matrix, $H \in \R^{d \times d}$, $\textrm{diag}(H) \in \R^d$ denotes the vector constituted by its diagonal elements $(H_{kk})_{1 \leqslant k \leqslant d}$. Classically, $\|\theta\|_1= \sum_{k=1}^d |\theta_k|$ stands for the $\ell_1$ norm of $\theta$. We denote  by $\[ 1,d \]$ the set of integers between $1$ and $d$. For vectors $\theta, \theta' \in \R^d$, $\theta \odot \theta'$ stands for the vector $(\theta_k \theta'_k)_{1 \leqslant k \leqslant d}$ and we define the square of a vector w.r.t this dot product: $\theta^2 = \theta \odot \theta$. Similarly $\log \theta$ and $\exp \theta$ (or $\e^\theta$) denote respectively the vectors $(\log \theta_k)_{1 \leqslant k \leqslant d}$ and $(\exp \theta_k)_{1 \leqslant k \leqslant d}$. For $\alpha, \alpha' \in \R$, the notation $\alpha \wedge \alpha'$ denotes the minimum between $\alpha$ and $\alpha'$. $\mathcal{N}(\mu, \sigma^2)$ is a Gaussian law of mean $\mu$ and variance $\sigma^2$. For any subset $S$ of $\[ 1,d \]$ of cardinal $|S|$, and vector $v \in \R^d$, we will denote by $v_S \in \R^{|S|}$, the vector $(v_k)_{k \in S}$. For any $v \in \R^d$, we will often write $v = [v_{S}, v_{S{{^{c}}}}]$, where $S^c = \[ 1,d \] \setminus S$. For the sake of clarity we also denote sometimes the $k$-th component of the vector $v$ by $[v]_k$ instead of $v_k$ as traditionally. Finally, $0_{\R^d}$ and $\iind{} \in \R^d$ denote respectively the vectors of zeros and ones.

\section{Setup and preliminaries}
\label{sec:setup}

The aim of the study is to show that the geometry of the noise induced by SGD can bias the dynamics towards sparse data interpolators. To precisely support this claim, we introduce now the model and the algorithms we consider. 

\subsection{Overparametrised noiseless sparse regression} We consider a linear regression problem with inputs-outputs $(\Mx_i, \My_i)_{1 \leqslant i \leqslant n}$ in $\R^d \times \R$, and loss function
\begin{equation}\label{eq:Lbeta}
  \ML(\beta):= \frac{1}{4n} \sum_{i=1}^n \left(\langle \beta, \Mx_i \rangle -  \My_i\right)^2 = \frac{1}{4n}\|\MX \beta - \My\|_2^2,
\end{equation}
where $\MX \in \R^{n \times d}$ is the data matrix whose rows are the data vectors $(\Mx_i^\top)_{1 \leqslant i \leqslant n}$ and $\My\in \R^n$ is the vector of outputs $(\My_i)_{1 \leqslant i\leqslant n}$. We study the overparametrised setting $d\geqslant n$ and assume that there exists at least one \textit{nonnegative} interpolating parameter which perfectly fits the training set, namely:
\begin{assumption}[Existence of nonnegative interpolators]
\label{ass:exist_nonneg_interp}
The set $\Interp := \{\beta \in \R_+^d: \ML(\beta)=0\}$ is nonempty.
\end{assumption}
Since we work in the overparametrised setting, the set $\Interp$ may in principle be a high-dimensional polyhedron. We now provide assumptions on $\MX$ and $\My$ ensuring the existence and characterisation of particular \textit{sparse} interpolators. The first one is very mild and ensures that the vector $\Mh := \textrm{diag}(\MX^\top \MX) \in \R_+^d$ has only positive coordinates. 
\begin{assumption}[No degenerate coordinate]
\label{ass:no_degenerate_coordinate}
The data matrix $\MX$ has no identically $0$ column.
\end{assumption}
For the following, let $S$ be a nonempty subset of $\[ 1, d\]$ with cardinality $s$ and define $\MX_S \in \R^{n \times s}$ as the matrix constituted by the coordinates of the inputs $(\Mx_i)_{1\leqslant i\leqslant n}$ solely in $S$. We use a similar notation for the vector $\Mh_S \in \R_+^s$. Furthermore, the support of any $\beta \in \R_+^d$ is denoted and defined by $\supp(\beta) := \{k \in \[1,d\]: \beta_k>0\}$, and we introduce the set $\mathcal{S}_+ := \{\supp(\beta), \beta \in \Interp\}$.
\begin{lemma}[Domination condition]\label{lem:condition_mystere}
  Under Assumptions~\ref{ass:exist_nonneg_interp} and~\ref{ass:no_degenerate_coordinate}, there is at most one $S \in \mathcal{S}_+$ such that either $S = \emptyset$, or $\MX_S^\top \MX_S$ is invertible and the domination condition
  \begin{equation}\label{eq:condition_mystere}
    \Mh_{S^c} > \MX^\top_{S^c}\MX_S(\MX^\top_S\MX_S)^{-1}\Mh_S
  \end{equation}
  holds, where the latter inequality is understood coordinatewise in $\R^{d-s}$.
\end{lemma}
Lemma~\ref{lem:condition_mystere} is proved in Appendix~\ref{app:Lasso}. Let us brefiely discuss some of its consequences. First, when $S \neq \emptyset$, the invertibility condition on $\MX_S^\top \MX_S$ implies that $|S| \leqslant n$, which gives that the interpolator associated with $S \in \mathcal{S}_+$ is at worse $n$-sparse. To understand the uniqueness property of $S$ stated in the Lemma let us detail the case with one data point in dimension two.

\begin{example}[Case $n=1, d=2$]
Let us set $\Mx=(\Mx^1, \Mx^2) \in \R^2$ and e.g. $\My=1$. Clearly \ref{ass:exist_nonneg_interp} and \ref{ass:no_degenerate_coordinate} are verified if and only if $\Mx^1, \Mx^2 \neq 0$ and either $\Mx^1 > 0$ or $\Mx^2 > 0$. If either one of the coordinate at least is negative, then the domination condition Eq.~\eqref{eq:condition_mystere} is always fulfilled and uniqueness of $S$ is obvious. Now suppose that $\Mx^1,\Mx^2  > 0$: if $S=\{1\}$, then Eq.~\eqref{eq:condition_mystere} is equivalent to $\Mx^2>\Mx^1$ and in the case $S=\{2\}$, Eq.~\eqref{eq:condition_mystere} is true if and only if $\Mx^1>\Mx^2$: hence, one possibility rejects the other one and uniqueness of $S$ holds. Finally, in the particular case where $\Mx^1=\Mx^2$, Eq.~\eqref{eq:condition_mystere} is never satisfied; we exclude this pathological case thanks to the following assumption.
\end{example}

\begin{assumption}[Existence of the ground-truth]\label{ass:mystere}
  There exists exactly one set $\Sstar \in \mathcal{S}_+$ which satisfies the conditions of Lemma~\ref{lem:condition_mystere}.
\end{assumption}
Notice that $\Sstar=\emptyset$ if and only if $\My=0$. Note additionally that under Assumption~\ref{ass:mystere}, we may define $\beta^\*$ as the unique element of $\Interp$ with support $\Sstar$ and when $\Sstar \neq \emptyset$, we can write explicitly $\beta^\*= [(\MX^\top_{\Sstar}\MX_{\Sstar})^{-1}\MX^\top_{\Sstar} \My , 0_{S^c}]$. We shall keep the notation $\beta^\*$, $\Sstar$ throughout the remainder of this article and call $\beta^\*$ the \textit{ground-truth estimator} of the model. In case the output $\My$ has been generated by a sparse nonnegative vector $\beta^S$ with data matrix $\MX$, i.e. $\MX \beta^S = \My$, and a support $S$ that satisfy the conditions of Lemma~\ref{lem:condition_mystere}, then this latter Lemma implies that $\beta^S = \beta^\*$. The aim of Assumption~\ref{ass:mystere} is to ensure the existence of such an interpolator. In the sparse recovery literature, the conditions asked in Lemma~\ref{lem:condition_mystere} are related respectively to the \textit{subinvertibility} and \textit{mutual incoherence} conditions. We refer to page 219 of \cite{wainwright2019high} for further discussions on these.

%

\subsection{Architecture and algorithm}
\paragraph*{A two-homogeneous reparametrisation.}
We reparametrise the linear prediction $\Mx \mapsto \langle \beta  , \Mx\rangle$ using the nonlinear parametrisation  $ \Mx \mapsto \langle \theta^2  , \Mx\rangle$, where the square of the vector $\theta$ stands for the coordinatewise square. This $2$-positive homogeneous model can be viewed as a simple linear network with only pairwise connections and is often used as a first step towards understanding more general neural networks~\citep{woodworth2020kernel,li2022happens}. It is worth noting that the parametrisation $\beta = \theta_+^2 - \theta_-^2$ could be considered in order to attain negative values. It would only make the analysis more technical and therefore we prefer to restrict ourselves to the simpler setting. By a slight abuse of notation we rewrite the training loss~\eqref{eq:Lbeta} as a function of $\theta$ as
\begin{equation}
\label{eq:loss_square}
\ML(\theta):= \frac{1}{4n} \sum_{i = 1}^n \left(\langle \theta^2, \Mx_i \rangle -  \My_i\right)^2.
\end{equation}
Even if the overall function space expressivity has not changed, the reparametrisation makes the least-square problem associated with  Eq.~\eqref{eq:loss_square} nonconvex. Thus minimising it with gradient based procedures is not guaranteed to converge to global optima anymore.

\paragraph*{Label noise gradient descent.}
We minimise the training loss in Eq.~\eqref{eq:loss_square} with GD and \textit{Label Noise} (LNGD). Namely, at each gradient step $t>0$, we deliberately add a random noise $\xi(t)\sim \mathcal{N}(0, \Mdelta I_n)$  to the label $\My$. The algorithm is started from  $\theta_0 $ and used with a constant step size $\gamma >0$. Noting explicitly the loss with input $\MX$ and output $\My$ as $\ML(\theta;\MX,\My)$, the update rule corresponds to
\begin{equation}
\label{eq:GD_label_noise}
\begin{split}
  \theta(t+1) &= \theta(t) - \gamma \nabla_\theta \ML(\theta(t);\MX,\My+\xi(t))\\
  &=\theta(t) -\frac{\gamma}{n} \left[ \MX^\top \left( \MX \theta(t)^2 - \My\right)\right] \odot \theta(t) + \frac{\gamma}{n} \left[ \MX^\top \xi(t)\right] \odot \theta(t),
\end{split}
\end{equation}
where $\odot$ stands for the coordinatewise product between two vectors.
In the previous literature, LNGD was often studied together with SGD. 
However the stochasticity coming from the sampling procedure of SGD is rapidly negligible compared to the one triggered by the label noise. Indeed the label noise SGD update writes
\begin{equation*}
\begin{split}
  \theta(t+1) &= \theta(t) - \gamma (\langle \theta(t)^2, \Mx_{i(t)}\rangle - \My_{i(t)}) \Mx_{i(t)} \odot \theta(t)\\
  & =  \theta(t) -\frac{\gamma}{n} \left[ \MX^\top \left( \MX \theta(t)^2 - \My\right)\right] \odot \theta(t) + \frac{\gamma}{n} \left[ \MX^\top (\xi(t) + \varepsilon_{i(t)}) \right] \odot \theta(t),
\end{split}
\end{equation*}
where for all $t>0$, $i(t)$ is sampled from the uniform distribution over $\[ 1 , n \]$, and we have defined $\varepsilon_{i(t)} :=\E_{i(t)} [(\langle \theta(t)^2, \Mx_{i(t)} \rangle - \My_{i(t)} )e_{i(t)}] - (\langle \theta(t)^2, \Mx_{i(t)} \rangle - \My_{i(t)} )e_{i(t)}$ where $(e_i)_{1 \leqslant i \leqslant n}$ is the canonical basis of $\R^n$.
With this notation, $\varepsilon_{i(t)}$ corresponds to the SGD noise: it is multiplicative, crucially vanishes at optimum and is rapidly negligible or comparable to $\sqrt{\Mdelta}$. Thus there should be no qualitative difference between SGD and GD when both are used with label noise (see Figure~\ref{fig:pink_flamingo} for an empirical validation of this fact). 
To understand the LNGD dynamics on the nonconvex objective Eq.~\eqref{eq:loss_square}, we resort to its continuous time model.  This approach has the advantage of leading to clean calculations while comprehending the complexity of the model.

\subsection{Label noise stochastic gradient flow}

\paragraph*{Continuous time stochastic dynamics modelling.}
Continuous time modelling of sequential processes provides a large set of tools, such as differential calculus, which are valuable when trying to understand the dynamics of a process.
For this reason, many recent works have considered gradient flows with the aim of grasping the behaviour of gradient descent on complex nonconvex problems such as neural networks training.  
However the modelling of stochastic dynamics is more demanding.
Indeed  these dynamics are better modelled by stochastic processes which are solutions of stochastic differential equations (SDEs): $\dd \theta(t) = b(t,\theta(t)) \dd t + \sigma(t,\theta(t))\dd B(t)$, where $(B(t))_{t \geqslant 0}$ is a standard Brownian motion.
For a proper model, the drift term $b$ and the noise  $\sigma$ need to be set in a particular manner:
\begin{enumerate}[label=(\roman*), topsep=2pt, parsep=0pt]
\item The drift term $b$ should match the negative gradient: $b=-\nabla \ML$.
\item The noise covariance $\sigma\sigma^\top(t,\theta)$ should match $\mathrm{Cov}[ \frac{\gamma}{n} \left[ \MX^\top \xi \right] \odot \theta(t) | \theta(t) = \theta ]$, where we have set $\xi \sim \mathcal{N}(0, \Mdelta I_n)$ independent of $\theta(t)$.
\item The noise should belong to the correct space, i.e. to the manifold $\{\theta \odot \mathrm{span}[\MX^\top] , \textrm{ for } \theta \in \R^d \}$.
\end{enumerate}

\paragraph*{Stochastic process model.} 
Following these rules, we propose the following SDE to model LNGD in continuous time. This gives the label noise gradient flow (LNGF)
\begin{equation}
\label{eq:SDE_model}
  \dd \theta(t) = -\frac{1}{n}[\MX^\top (\MX\theta(t)^2 - \My) ] \odot \theta(t) \dd t + \frac{\sqrt{\Mdelta \gamma}}{n} \   \theta(t) \odot [\MX^\top \dd B(t)],
\end{equation}
with $(B(t))_{t \geqslant 0}$ a standard Brownian motion in $\R^n$. 
Since LNGD is the Euler-Maruyama discretisation with step-size $\gamma$ of the SDE \eqref{eq:SDE_model}, the SDE and the discrete models match perfectly for infinitesimal step-sizes (up to first order terms in $\gamma$). The model is said to be \textit{consistent}.
We also note that the same SDE is obtained for any label noise distribution $\xi$ with zero-mean and $\Mdelta$ times identity covariance. We simply assume that the injected label noise is Gaussian for clearness of exposition.  
In the same way, we could consider time-dependent covariance $\Sigma(t) \in \R^{n \times n}$ by replacing the noise term in the SDE by $\frac{\sqrt{\gamma}}{n} \   \theta(t) \odot [\MX^\top \Sigma(t)^{1/2} \dd B(t)]$. 

\paragraph*{Drift-Variance tradeoff.}
%
%
For the sake of clarity, we introduce the following renormalised variables: $\RX = \MX/\sqrt{n}, \Ry = \My/\sqrt{n}, \Rh = \Mh/n, \Rdelta = \gamma \Mdelta / n$. The SDE model we study reads
\begin{equation}
\label{eq:SDE_model_normalized}
  \dd \theta(t) = \underbrace{-[\RX^\top (\RX\theta(t)^2 - \Ry) ] \odot \theta(t)}_{\textrm{drift term}} \dd t+ \underbrace{\sqrt{\Rdelta} \   \theta(t) \odot [\RX^\top \dd B(t)]}_{\textrm{noise term}}.
\end{equation}
As often with state-dependent noise SDEs, there is a competition between the noise and the drift components.
On the one hand, considering the drift term alone amounts to only take into account the \textit{gradient flow}. As analysed by~\citet{woodworth2020kernel,NEURIPS2020_e9470886}, the dynamics is, in that case, driven to a certain interpolator related to the initialisation.
On the other hand, the noise term acts as a multiplicative shrinking force akin to the one in the geometric Brownian motion $\dd S(t) = \mu S(t) \dd t + \sigma S(t) \dd W(t)$.
For this dynamics, the noise can counter the repulsive force of the drift and drive  $S(t)$ to $0$ almost surely if its scale satisfies $\sigma^2 > 2\mu$.
Thus, if the noise $\Rdelta$ dominates the dynamics: $\Rdelta \gg \sup_{t\geqslant 0} \|\RX^\top (\RX\theta(t)^2 - \Ry)\|_\infty $, the process is similarly driven to $0$ almost surely. We note that this argument is at the crux of the analysis of~\citet{haochen2021shape}.
However, when the noise level $\Rdelta$ is not infinitely large, the drift and the noise balance each other out and the dynamics becomes much more intricate to analyse.

\subsection{Hidden mirror flow structure and Lasso}
\label{subsec:hidden_mirror}

\paragraph*{Itô calculus and hidden mirror flow.} 
Let us recall that, for such reparametrised model, when $\theta$ follows a gradient flow $\dd\theta(t) = -\nabla \ML(\theta(t))\dd t$, then the corresponding iterate $\beta(t) = \theta(t)^2$ follows a mirror descent with potential defined through $\nabla\psi(\beta) = \log(\beta)$, where the $\log$ is taken componentwise~\citep{ghai2020exponentiated}.
This result is easily obtained using the chain rule on $\nabla\psi(\beta(t))$.
This hidden mirror structure is then used to describe the implicit bias of such gradient flows.
%
It turns out that the exact same procedure can be done based on Itô calculus (i.e. chain rule for stochastic processes) to exhibit a stochastic mirror flow. 
Indeed, setting $\beta(t) = \theta(t)^2$, where $(\theta(t))_{t \geqslant 0}$ is the solution to~\eqref{eq:SDE_model_normalized}, we have
\begin{equation}\label{eq:sdebeta}
  \dd \beta(t) = \beta(t) \odot \left[-2 \RX^\top(\RX \beta(t) - \Ry) + \Rdelta \Rh\right]\dd t + 2 \sqrt{\Rdelta} \beta(t) \odot [\RX^\top \dd B(t)],
\end{equation}
and this entails the following stochastic differential equation for $\log(\beta)$
\begin{equation}
\label{eq:sto_mf}
\dd \log \beta(t) = -\left(2 \RX^\top (\RX \beta(t) - \Ry)  + \Rdelta \Rh\right) \dd t+ 2\sqrt{\Rdelta} \RX^\top \dd B(t).
\end{equation}
This observation has already been used to understand the implicit bias due to the sampling noise of SGD~\citep{pesme2021implicit}, where the authors resort to the unsigned parametrisation ($\beta = \theta^2_+ - \theta^2_-$) changing the mirror map to the $\mathrm{argsh}$; without affecting qualitatively the following discussion.
\paragraph*{Stochastic mirror on the weighted lasso.}
The equation \eqref{eq:sto_mf} can be explicitly rewritten as a stochastic \textit{mirror-like} flow
\begin{equation}
\label{eq:stochastic_mirror_flow}
  \dd \nabla \psi (\beta(t)) = -\nabla \RL_{\Rdelta}(\beta(t))  \dd t+ 2\sqrt{\Rdelta} \RX^\top \dd B(t),
\end{equation}
with
\begin{equation}\label{eq:Ldelta}
  \RL_{\Rdelta}(\beta) := \|\RX\beta - \Ry\|_2^2 + \Rdelta \langle \Rh, \beta \rangle, \qquad \nabla \psi(\beta) = \log \beta.
\end{equation}
The objective $\RL_{\Rdelta}$ is importantly related to the celebrated \textit{weighted} Lasso problem:
\begin{equation}\label{eq:WL}
  \tag{$\text{WL}_{\Rdelta}$}
  \min_{\beta \geqslant 0} \ \RL_{\Rdelta}(\beta),
\end{equation}
where the quadratic loss is regularised by a weighted $\ell_1$-norm with weight $\Rh$ and regularisation parameter $\Rdelta$.
Hence, the change of variable together with the mirror interpretation provides the following intuition:
%
\begin{center}
\textit{The label noise gradient flow can be cast as a stochastic mirror flow on the weighted Lasso with weight $\Rh$ and regularisation parameter $\Rdelta$.}
\end{center}
This picture serves as the main guideline for our study.
However, the stochasticity present in the mirror prevents from directly applying mirror-based optimisation techniques.
Instead, our approach is based on a direct analysis of the dynamics of $\beta(t)$.  Our results are gathered in the next section. 


\section{Main results}
\label{sec:results}

\subsection{The recovery problem for the weighted Lasso} As we have seen in the previous section, the stochastic dynamics Eq.~\eqref{eq:SDE_model_normalized} is intimately related to a mirror gradient optimisation on the weighted $\ell_1$ least-square problem~\eqref{eq:WL}. Hence, to understand the behaviour of the dynamical problem, it seems natural to understand first the properties of the weighted Lasso. This problem is referred to as the \textit{variable selection consistency of the Lasso} in the literature (see e.g. Section~7.5.1 of \cite{wainwright2019high}). Note that due to the weight $\Rh$ in the $\ell_1$ norm and the positivity constraint in~\eqref{eq:WL}, the characterisation of our problem is not immediately implied by standard theorems. We note that, since the function $\RL_{\Rdelta}$ is convex, a vector $\beta \in \R_+^d$ is a solution to~\eqref{eq:WL} if and only if there exists $\mu \in \R_+^d$ which satisfies the Karush--Kuhn--Tucker condition
\begin{equation}\label{eq:KKT}
  \tag{$\text{KKT}_{\Rdelta}$}
  2\RX^\top(\RX\beta-\Ry) + \Rdelta\Rh = \mu, \qquad \langle \mu, \beta\rangle = 0.
\end{equation}
\vspace*{-0.5cm}
\begin{theorem}
\label{thm:Lasso}
  Let Assumptions~\ref{ass:exist_nonneg_interp}, \ref{ass:no_degenerate_coordinate} and~\ref{ass:mystere} hold, and let $\beta^\*, \Sstar$ be defined thereby.
  \begin{enumerate}
    \item If $\Sstar=\emptyset$, then for any $\Rdelta>0$, the pair $\betaL:=0$, $\muL:=\Rdelta\Rh$ satisfies the condition~\eqref{eq:KKT}.
    \item Otherwise, set
  \begin{align*}
    \Rdelta_- &:= \sup\{\Rdelta>0: \forall k \in \Sstar, \beta^\*_k > \Rdelta [(2\RX_{\Sstar}^\top \RX_{\Sstar})^{-1} \Rh_{\Sstar}]_k\},\\
    \Rdelta_+ &:= \inf\{\Rdelta>0: \forall k \in \[1,d\], \Rdelta\Rh_k > 2 [\RX^\top \Ry]_k\}.
  \end{align*}
\begin{enumerate}[label=(\roman*)]
\item Standard noise regime: if $\Rdelta < \Rdelta_-$, then the pair
\begin{align*}
  \betaL &:= [\beta^\*_{\Sstar} - \Rdelta(2\RX_{\Sstar}^\top \RX_{\Sstar})^{-1} \Rh_{\Sstar}, 0_{{\Sstar}^c}]^\top,\\
  \muL &:= [0_{\Sstar}, \Rdelta(\Rh_{{\Sstar}^c} - \RX_{{\Sstar}^c}^\top \RX_{\Sstar} (\RX_{\Sstar}^\top \RX_{\Sstar})^{-1}\Rh_{\Sstar})]^\top,
\end{align*}
satisfies the condition~\eqref{eq:KKT}.
\item Large noise regime: if $ \Rdelta > \Rdelta_+$, then the pair $\betaL := 0$, $\muL := \Rdelta \Rh - 2 \RX^\top \Ry>0$ satisfies the condition~\eqref{eq:KKT}.
\end{enumerate}
\end{enumerate}
In all cases, $\betaL$ is the unique solution to~\eqref{eq:WL}.
\end{theorem}
Let us comment this theorem. First, uniqueness for (i) and (ii) is nontrivial because the~\eqref{eq:WL} problem is not strongly convex on account of the degeneracy caused by the affine space $\beta^\* + \ker \RX$. Second, (i) tells us that support recovery is perfectly achieved in the case of a standard noise level. The inequality $\Rdelta < \Rdelta_-$ should be interpreted as a high enough \textit{signal to noise ratio} that allows the estimator to recover the support of the ground truth $\beta^\*$. To give an order of magnitude, in case of standard i.i.d. Gaussian data inputs, $\RX_{\Sstar}^\top \RX_{\Sstar}\sim I_{\Sstar}$, $\Rh_{\Sstar} \sim \iind{_{\Sstar}}$ and $\RX_{{\Sstar}^c}^\top \RX_{\Sstar} \beta_S \sim 0_{{\Sstar}^c}$. Hence,  $\delta_- \sim 2\beta^\*_{\text{min}}$ and $\delta_+ \sim 2\beta^\*_{\text{max}}$ where $\beta^\*_{\text{min}}$ and $\beta^\*_{\text{max}}$ are respectively the minimum and the maximum value of $\beta^\*_{S^\*}$ . The signal to noise ratios $\Rdelta/\beta^\*_{\text{min}}$ and $\Rdelta/\beta^\*_{\text{max}}$ rule the difference between the two regimes. Third, note that although our two regimes describe most of the cases we are interested in, a band of noise scale is not treated by Theorem~\ref{thm:Lasso}: this is typically the case when $\Rdelta\in [2\beta^\*_{\text{min}}, 2\beta^\*_{\text{max}}]$ in which case the solution could have support outside $S\*$. The proof of this result is classical: we exhibit \textit{ad-hoc} explicit solution to condition~\eqref{eq:KKT} and show it is unique. This reasoning is referred to as the \textit{Primal–dual witness construction} by~\citet[p.223]{wainwright2019high} and is derived in Appendix~\ref{app:Lasso}. Finally, note that the case $\Sstar=\emptyset$ can be seen as a subcase of the large noise regime, as when $y = 0$, we have $\delta_+ = 0$ and the solutions match. The distinction between standard and large noise regime finds an echo in the next section.

\subsection{SDE convergence results}

Before presenting the main results, we first recall the most important observations stated in the previous sections. First, as drift and noise parts are locally Lipschitz continuous, for any initial condition $\theta(0) \in \R^d$, Eq.~\eqref{eq:SDE_model_normalized} has a unique strong solution, which is defined up to some explosion time $\tau_\infty$~\citep{khasminskii2012stability}. Throughout the sequel, we shall work with a fixed initial condition $\theta(0)$ such that $\beta(0)=\theta^2(0)>0$. Our results will then entail that in both regimes introduced in the previous section, $\tau_\infty=+\infty$, almost surely. Remarkably, Eq.~\eqref{eq:stochastic_mirror_flow} of Section~\ref{subsec:hidden_mirror} shows that it can be cast as a Lasso stochastic mirror flow on the linear predictor $\beta = \theta^2$. 
Then, conditions are given in Theorem~\ref{thm:Lasso}  under which support recovery is achieved by the minimiser $\betaL$ of the weighted Lasso program~\eqref{eq:WL} $\!$. 
Here, we naturally distinguished between two different regimes: (i) the large noise regime where the best predictor is uniformly zero (Section~\ref{subsub:large_noise}) and (ii) the standard noise regime, when the signal to noise ratio is high enough to allow for support recovery (Section~\ref{subsub:standard_noise}).

Our main results show that $(\beta(t))_{t \geqslant 0}$ recovers perfectly the support of $\betaL$ in these two regimes, namely that, first, $\beta(t) \to 0 $ on ${S^\*}^c$. Second, on the support $S^\*$, in the long run, $\beta(t)$ fluctuates in a neighbourhood of size $\sqrt{\Rdelta}$ around $\betaL$ and hence of $\beta^\*$ (by Theorem~\ref{thm:Lasso}-(ii)).
To quantify the noise that remains inherently in $\beta(t)$, we introduce comparison processes that allow to precisely specify the scale of the fluctuations. 
In the large noise regime, it is ruled by Brownian fluctuations, whereas in the standard noise regime, it is ruled by a rapidly mixing process which concentrates around $\beta^\*$ at $\sqrt{\Rdelta}$ scale.

\subsubsection{The large noise regime}
\label{subsub:large_noise}
We first place ourselves in the large noise regime from Theorem~\ref{thm:Lasso}. We recall that we take the convention that this regime contains the case $S^\*=\emptyset$ (that is to say $\Ry=0$), for which we define $\Rdelta_+=0$.
In this regime, note that $\betaL$ is uniformly zero. In the following result, we  show that the stochastic gradient flow we consider goes to zero almost surely at exponential speed.
\begin{theorem}[Large noise regime convergence]
\label{thm:convergence_large_noise} Let the assumptions of Theorem~\ref{thm:Lasso} hold, with $\Rdelta>\Rdelta_+$. Recall that $\muL = \Rdelta \Rh - 2\RX^\top \Ry>0$. Then $\tau_\infty=+\infty$, almost surely; besides, there exists $C \geqslant 0$ depending on the data $\RX$, $\Ry$ and on $\beta(0)$, and a one-dimensional process $(\RZd(t))_{t \geqslant 0}$, which has the same law as $(4\Rdelta\|B(t)\|^2_2)_{t \geqslant 0}$,
such that almost surely, 
\begin{equation}
\forall t \geqslant 0, \qquad \beta(t) \leqslant C
 \exp\left(\|\RX^\top\| \sqrt{\RZd(t)} \iind{} -\muL t\right).
\end{equation}
In consequence,
\begin{equation}
  \lim_{t \to +\infty} \beta(t) = \betaL = 0, \qquad \text{almost surely.}
\end{equation}
\end{theorem}
The theorem shows quantitatively that all the coordinates of our predictor go to zero exponentially fast at some speed governed by $\Rdelta$ (through the variable $\muL$). Roughly speaking, as the fluctuations of the \textit{Bessel process} $\sqrt{\RZd(t)}$ are of order $\sqrt{t}$~\citep[Chapter XI]{revuz2013continuous}, we see that they are negligible in front of the linear deterministic term $-\muL t$.  These two terms in the exponential are quite reminiscent of the one dimensional geometric Brownian motion (GBW). Thus, we can see this regime as a multidimensional generalisation of the GBW. Note that the same idea has been expressed in~\cite{haochen2021shape} but with an unnecessarily large noise $\Rdelta \geqslant \mathcal{O}(\beta(0) d^2)$. To conclude on the almost sure convergence, a precise analysis is conducted in the proof using the law of the iterated logarithm~\citep[Corollary~1.12, Chapter~II]{revuz2013continuous}. The proof as well as the explicit value of $C$ can be found in Appendix~\ref{app:large_noise_regime}. 
%
\subsubsection{The standard noise regime}
\label{subsub:standard_noise}
In this regime, recall that by Theorem~\ref{thm:Lasso}, $\betaL$ has the same support as $\beta^\*$, denoted by $S^\* \subset \[1,d\]$. For technical reasons we need to assume that the initial condition $\beta(0)$ satisfies the condition that $\beta_{S^\*}(0) \not= \betaL_{S^\*}$, and moreover that $n \geqslant 2$. In the following result, we prove that the stochastic gradient flow we consider goes to zero almost surely at exponential speed on ${S^\*}^c$. We then show precisely that, on the support $S^\*$, the stochastic flow fluctuates in a region of size $\sqrt{\Rdelta}$ around $\betaL$. To quantitatively support this claim we introduce a comparison process $(\Xid(t))_{t \geqslant 0}$ which quantifies the scale of the fluctuations.
\begin{theorem}[Standard noise regime convergence]
\label{thm:standard_noise}
Let the assumptions of Theorem~\ref{thm:Lasso} hold, with $S^\*\not=\emptyset$ and $\Rdelta<\Rdelta_-$. Recall that $\muL_{{S^\*}^c}>0$. Then $\tau_\infty=+\infty$, almost surely; besides, there exists a one dimensional stochastic process $(\Xid(t))_{t \geqslant 0}$, such that the following assertions hold.
\begin{enumerate}[label=(\roman*), topsep=5pt, parsep=1pt]
\item On the support $S^\*$, almost surely,
\begin{equation}\label{eq:thm_support}
\forall t \geqslant 0, \qquad \betaL_{S^\*}\e^{-\|\RX_{S^\*}^\top\|\sqrt{\Xid(t)}} \leqslant \beta_{S^\*}(t)   \leqslant \betaL_{S^\*}\e^{\|\RX_{S^\*}^\top\| \sqrt{\Xid(t)}}.
\end{equation}
Furthermore, $\Xid(t)$ converges in distribution towards a random variable $\Xid_\infty$ which satisfies the concentration property
\begin{equation}\label{eq:concentration_xi}
  \forall u  > 0, \qquad \Pr\left(\Xid_\infty \geqslant u \Rdelta\right) \leqslant 6 \exp\left(-\sqrt{\frac{u \Rdelta}{\kappa^{\Rdelta}}}\right),
\end{equation}
for some $\kappa^{\Rdelta}>0$ such that $\kappa^{\Rdelta} = \mathcal{O}(\delta)$ when $\Rdelta$ is infinitesimally small. 
\item Outside the support $S^\*$, there exists $C \geqslant 0$, depending on $\RX$, $\Ry$ and $\beta(0)$, such that, almost surely, 
\begin{equation}\label{eq:thm_0}
\forall t \geqslant 0, \qquad \beta_{{S^\*}^c}(t) \leqslant C \exp\left(\|\RX_{{S^\*}^c}^\top\|\sqrt{\Xid(t)}\iind{{S^\*}^c} - \muL_{{S^\*}^c} t\right).
\end{equation}
This implies that
\begin{equation*}
  \lim_{t \to +\infty} \beta_{{S^\*}^c}(t) = 0, \qquad \text{almost surely.}
\end{equation*}
\end{enumerate}
\end{theorem}
We make the following remarks on the theorem. In the long term limit, for $\Rdelta$ small, Equation~\eqref{eq:concentration_xi} and the fact that $\kappa^{\Rdelta} = \mathcal{O}(\Rdelta)$ show that the random variable $\Xid(t)$ is of order $\mathcal{O}(\Rdelta)$. This concentration property follows from the fact that $\Xid_\infty$ can be expressed as the squared norm of a $n$-dimensional random variable, the law of which satisfies a Poincaré inequality with a constant which can be estimated thanks to the Laplace method. Therefore, on the support of $\beta^\*$, by Equation \eqref{eq:thm_support} we have, in the long term limit,
\begin{equation*}
  \beta_{S^\*}(t) \simeq \betaL_{S^\*} + \mathcal{O}\left(\sqrt{\Rdelta}\right) = \beta_{S^\*}^\*  +  \mathcal{O}\left(\sqrt{\Rdelta}\right),
\end{equation*}
since Theorem~\ref{thm:Lasso} indicates that $\betaL_{S^\*}-\beta_{S^\*}^\*$ is of order $\Rdelta$.

On the other hand, from Equation~\eqref{eq:thm_0} we see that the story outside the support is simpler: as in the large noise regime, the iterates go exponentially fast to $0$ at rate $\muL_{S^c} > 0$. Note however that the scale of the fluctuations is much smaller than is the large noise regime: it is given by the ergodic process $(\Xid(t))_{t \geqslant 0}$, and no longer by the Bessel process $(\RZd(t))_{t \geqslant 0}$. The reason is that in the standard noise regime, there is an \textit{effective} strong convexity effect that prevents the noise from truly fluctuating far from some deterministic trajectory. 

Precise derivations on $(\Xid(t))_{t \geqslant 0}$ and explicit constants can be found in Appendix~\ref{sub:standard_noise}.

\section{Dynamics description and proof sketch}
\label{sec:dynamics}

\subsection{Going beyond the mirror shape: a dual process}
\label{subsec:dual_processx}

\paragraph*{The time dependent mirror.} 
%

%
%

Previous analysis~\citep{gunasekar2018geometry,woodworth2020kernel,pesme2021implicit} on the implicit bias of such models rely on the crucial observation that the dual iterates  $\nabla \psi(\beta(t))$ belong to the linear span of the observations $\mathrm{span} (\RX^\top)$. 
However, in our case, the additional drift term does not belong to such a space: i.e. $\Rdelta \Rh \notin \mathrm{span} (\RX^\top)$. This is why we introduce a time dependent mirror map  $\nabla_\beta\psi(t,\beta) = \log (\beta \e^{\Rdelta \Rh t})$. By Eq.~\eqref{eq:sto_mf}, the process $(\beta(t))_{0 \leqslant t < \tau_\infty}$ then satisfies
\begin{equation*}
\dd \nabla_\beta\psi(t,\beta(t)) = -2 \RX^\top (\RX \beta(t) - \Ry) \dd t+ 2\sqrt{\Rdelta} \RX^\top \dd B(t),
\end{equation*}
which is a stochastic mirror-like descent with a geometry that depends on time. The vector $\nabla_\beta \psi(t,\beta(t))$ can then be decomposed in a unique way along $\mathrm{span} (\RX^\top)$ and $(\mathrm{span} (\RX^\top))^\perp  = \ker \RX$. 
More precisely, there exist $\Uz \in \ker \RX$ and $v(0) \in \R^n$ such that $\log\beta(0) = \Uz + \RX^\top v(0)$. We next consider the unique strong solution to the SDE
\begin{equation*}
  \dd v(t) = - 2 \left(\RX \exp(\RX^\top v(t) + \Uz - \Rdelta \Rh t) - \Ry\right) \dd t + 2 \sqrt{\Rdelta} \dd B(t),
\end{equation*}
and notice that $\exp(\Uz + \RX^\top v(t)-\Rdelta\Rh t)$ then satisfies the SDE~\eqref{eq:sdebeta}.
%
%
Therefore, by pathwise uniqueness, $\beta(t) = \exp(\Uz + \RX^\top v(t)-\Rdelta\Rh t)$ and both processes $\beta$ and $v$ share the same explosion time~$\tau_\infty$.

\paragraph*{Defining a dual process.} 

We next remove the constant term $2\Ry$ in the drift of $v(t)$ by defining $\vv(t):= v(t) - 2 \Ry t \in \R^n$. Denoting by $F(t,\vv):= \|\exp(\RX^\top \vv + \Uz - (\Rdelta \Rh - 2 \RX^\top \Ry)t)\|_1$, we remark that $(\vv(t))_{0 \leqslant t < \tau_\infty}$ satisfies the SDE
\begin{equation}\label{eq:SDE_v}
  \dd \vv(t) = - 2 \nabla_v F(t,\vv(t)) \dd t + 2 \sqrt{\Rdelta} \dd B(t).
\end{equation} 
We define as well the associated gradient flow $(\ww(t))_{0 \leqslant t < \tilde{\tau}_\infty}$:
\begin{equation}\label{eq:GF_w}
  \dd \ww(t) = -2 \nabla_v F(t,\ww(t)) \dd t,
\end{equation}
initialised similarly, that is $\ww(0) = \vv(0) = v(0)$, and defined up to some (deterministic) explosion time $\tilde{\tau}_\infty$. 
For $t\geqslant 0$ we also introduce,  with a slight abuse of notation, the map $\beta(t,\cdot)$ defined from $\R^n $ to $\R^d$ as $\beta(t,\vv):= \exp(\RX^\top \vv + \Uz - (\Rdelta \Rh - 2 \RX^\top \Ry)t)$. 
This map represents the unique way to go from the process $\vv(t)$, solution of the SDE~\eqref{eq:SDE_v}, to the associated process of linear predictors $\beta(t)$ that follows the SDE~\eqref{eq:sdebeta}.
Thus, we refer to $\vv(t)$ as \textit{the dual process} associated to $\beta(t)$.
It is also worth noting that $F(t,\vv)= \|\beta(t,\vv)\|_1$ is a convex function of the second variable. This equality emphasises the sparse promoting effect of the dynamics on the linear predictor.  Finally, an important quantity appearing in the exponential is the vector $$c:=\Rdelta \Rh - 2 \RX^\top \Ry.$$ In the two next sections, we show that, depending on its sign, the behaviour of Eq.~\eqref{eq:SDE_v} and \eqref{eq:GF_w} is very different.

\subsection{Large noise regime and the lack of strong convexity}
\label{subsec:large_noise}

In the setting of Theorem~\ref{thm:convergence_large_noise}, the vector $c$ coincides with $\muL$ defined in Theorem~\ref{thm:Lasso}, and thus satisfies $c > 0$.
Hence the noise dominates the dynamics and entails that all coordinates of $\beta(t,\vv(t))$ are shrunk by a $\e^{-\muL t}$ factor. Here, we proceed in two steps:
\begin{enumerate}[label=(\roman*)]
\item First, we show that the gradient flow~\eqref{eq:GF_w} on the dynamics of the process $(\ww(t))_{0 \leqslant t < \tilde{\tau}_\infty}$ is bounded.
Therefore, $\tilde{\tau}_\infty=+\infty$ and the associated linear estimator $\beta(t,\ww(t))$ goes to $0$ at rate~$\e^{-\muL t}$.
\item Then, we show that the stochastic process~\eqref{eq:SDE_v} on $\vv(t)$ fluctuates from the gradient flow $\ww(t)$ at a distance controlled by the norm of an $n$-dimensional Brownian motion, i.e. a Bessel process. Hence, as Brownian fluctuations are of order $\sqrt{t}$, these are negligible in front of the $-\muL t$ gradient flow decay rate, which yields in particular $\tau_\infty=+\infty$. 
\end{enumerate}
We put emphasis on the fact that this regime suffers form a lack of strong convexity. This is reminiscent of optimizing the Lasso outside of the support of the sparse ground-truth $\beta^\*$ (see the discussion section 2.2 of \citet{argawal2012fast} on \textit{Restricted strong convexity and smoothness}).

\subsection{Standard noise regime and strong convexity}
\label{subsec:small_noise}

\paragraph*{Failing of the previous reasoning.}
In the setting of Theorem~\ref{thm:standard_noise}, it is expected that the dual gradient flow \eqref{eq:GF_w} drives the primal iterate  $\beta(t,\ww(t))$ towards the solution of the weighted Lasso $\betaL$. 
The strategy of comparison with the gradient flow, outlined in the previous section, only works for the coordinates of $\beta(t,\vv(t))$ outside of the support $S^\*$ of $\betaL$.
For these coordinates, the fluctuations are still dominated by the linear exponential shrinking.
Unfortunately, when applying the same technique for coordinates in the support $S^\*$, the fluctuations totally blur the sparse recovery problem. 
To overcome this obstacle, we leverage the strong convexity of the function  $\vv \mapsto F(t,\vv)$ on the support $S^\*$, as explained in the following paragraph.

\paragraph*{A linear classification problem structure.}
The reasoning rests on the following observation: the dual gradient flow \eqref{eq:GF_w} is similar to the one that solves a linear classification problem with the exponential loss \textit{on the data points given by the coordinates of the $(x_i)_{1 \leqslant i \leqslant n}$}.
More precisely, let us define the $d$ data $(\Rz_k)_{1 \leqslant k \leqslant d} \in (\R^{n})^d$ such that $\RZ = \RX^\top$, i.e for all $k \in \[1,d\]$, $i \in \[1,n\]$ $[\Rz_k]_i = [\Rx_i]_k$. Then \eqref{eq:GF_w} is equivalent to
\begin{equation*}
\dd \ww(t) = - 2 \RZ^\top \exp(\RZ \ww(t) + \Uz - ct) \dd t =  - 2 \sum_{k = 1}^d \Rz_k \exp\left(\langle \Rz_k, \ww(t)\rangle  + [\Uz]_k - c_k t\right) \dd t,
\end{equation*}
and the same goes for \eqref{eq:SDE_v}. Note that this observation is reminiscent of the primal-dual analysis presented in~\cite{ji2021characterizing}.
We leverage this equivalence and show that the selection of the support vectors (i.e,  the samples with the smallest margin)   of  separable linear classification problem~\citep{soudry2018implicit} enables to recover the true support $S^\*$.
In this aim, we show that the only terms that matter in the exponential sum are those in the set $S'$ of indices $k$ for which $\langle \Rz_k, \ww(t) \rangle \sim c_k t$ when $t \to +\infty$, 
since the other terms go to zero exponentially fast. 
Hence, $\beta(t,\ww(t))$ is asymptotically supported by the coordinates of $S'$. As in the separable classification problem~\citep{soudry2018implicit}, it turns out that we can totally identify $S'$ as in fact $S' = S^\*$. Finally, we show that the fluctuations due to the stochastic remaining terms are controlled by the strong convexity upon the support and are of order $\sqrt{\Rdelta}$.

\section{Experiments}
\label{sec:experiments}

We consider the following sparse regression setup for our experiments. We choose $n = 40$, $d = 100$ and randomly generate a sparse model $\beta^S$ such that the cardinality of its support is $s = 4$. We generate Gaussian features as $\Mx_i \sim \mathcal{N}(0, I_d)$, labels as $\My_i = \langle \beta^S, \Mx_i \rangle $ and check that this model satisfies the assumptions required in Theorem~\ref{thm:Lasso}. We consider four different algorithms: the two first are the one we model i.e, Gradient descent + label noise and Stochastic gradient descent + label noise, and for comparison we also considered the Gradient descent and the Stochastic gradient descent. They are initialised at the same point and run with the same step size $\gamma=0.1$ and noise $\Mdelta = 10^{-3}$. 

The experiment presented in Figure~\ref{fig:pink_flamingo} perfectly illustrates Theorem~\ref{thm:convergence_large_noise} (right plot) and Theorem~\ref{thm:standard_noise} (center plot). Overall, sparse recovery is achieved up to scale $\sqrt{\Rdelta} = \sqrt{\Mdelta \gamma / n} \sim 1.6 \cdot 10^{-3}$. Note that there is no qualitative difference between SGD + label noise (shadowed orange) and GD + label noise (orange). This validates our SDE model.
\begin{figure}[ht]
\centering
\begin{minipage}[c]{.32\linewidth}
\includegraphics[width=\linewidth]{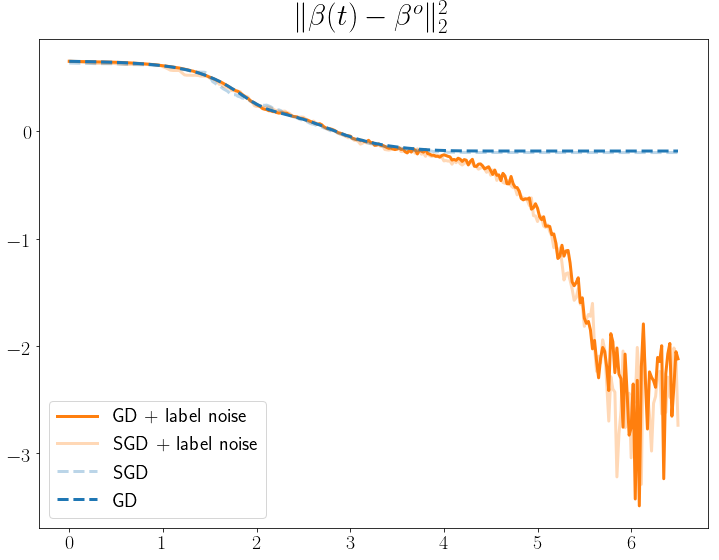}
   \end{minipage}
   \hspace*{0pt}
   \begin{minipage}[c]{.32\linewidth}
\includegraphics[width=\linewidth]{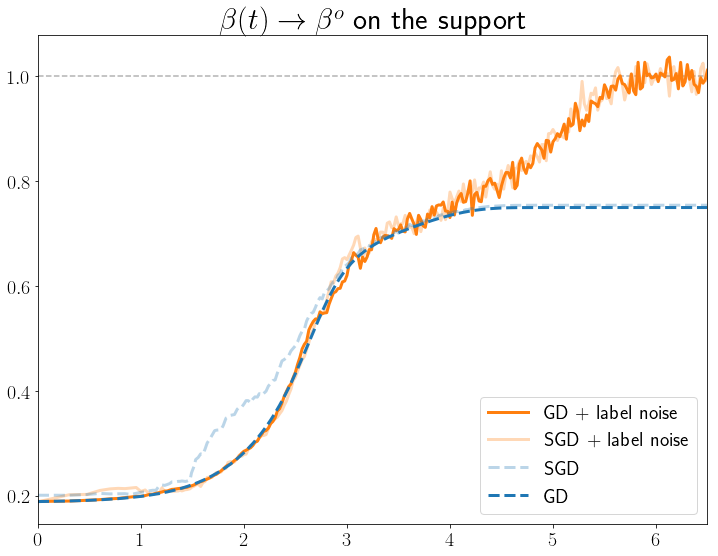}
   \end{minipage}
   \hspace*{0pt}
  \begin{minipage}[c]{.32\linewidth}
\includegraphics[width=\linewidth]{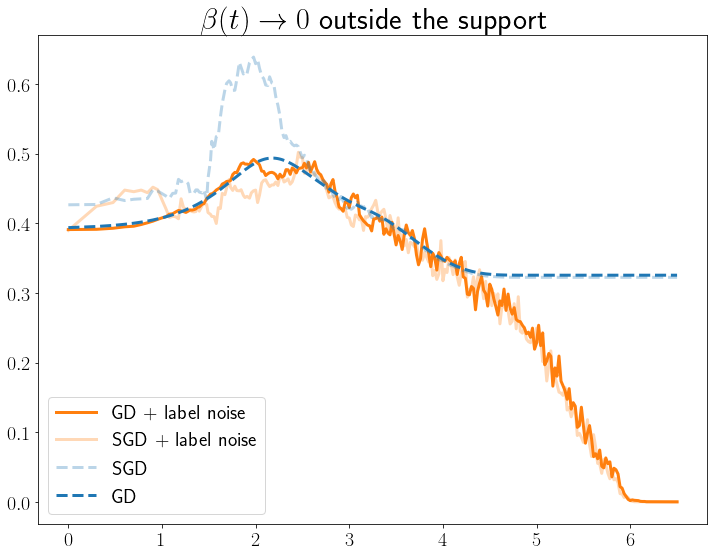}
   \end{minipage}
  \caption{Sparse regression on synthetic data. 
  \textit{Left}: Square error on the parameters. GD and SGD do not recover the sparse truth whereas GD + label noise and SGD + label noise achieve recovery up to scale $\sqrt{\Rdelta}$.
  \textit{Center}: Recovery of the support in case of noisy label dynamics. A unique and prototypical coordinate is displayed here.
  \textit{Right}: Convergence to $0$ outside the support in case of noisy label dynamics for a prototypical coordinate.
  }
     \label{fig:pink_flamingo}
\end{figure}

\section{Conclusion}

In this paper, we have shown that, for quadratically parametrised predictors, the label noise gradient descent solves implicitly a weighted Lasso optimisation program. Hence, this stochastic descent is able to perfectly recover the support of the sparse ground-truth when the injected noise is not too large. In contrast with previous works, we derived precise nonasymptotic results, both it terms of time and noise. Surprisingly, the heart of the proof is based on an equivalence between the selection of support vectors for a classification problem and the one of the nonzero coordinates of the sparse ground-truth. Whether this is only a technical equivalence or a deep relationship could be of great interest. Also, we characterise the equivalence between our label noise dynamics and an optimisation problem enforcing sparsity. Whether it can be done for other architecture remains an interesting open question.

%
%
%
%
%
%
%
%
%
%
%
%
%
%
%
%
\acks{N.F. dedicates this work to his late dear friend O.E.. The work of J.R. is supported by the projects ANR EFI (ANR-17-CE40-0030) and ANR QuAMProcs (ANR-19-CE40-0010) from the French National Research Agency.}

\clearpage

\bibliography{Label_noise_lasso}

\clearpage

\appendix

{\LARGE \noindent \bfseries Appendix}

\section{Proof of Lemma~\ref{lem:condition_mystere} and Theorem~\ref{thm:Lasso}}
\label{app:Lasso}

We first recall the notation $\mathcal{S}_+$ introduced after Assumption~\ref{ass:exist_nonneg_interp}, and set
\begin{align*}
  \mathcal{S}'_+ &:= \left\{S \in \mathcal{S}_+: \text{$S \not= \emptyset$, $\MX_S^\top \MX$ is invertible, and $\Mh_{S^c} > \MX_{S^c}^\top \MX_S (\MX_S^\top \MX_S)^{-1}\Mh_S$}\right\}\\
  &= \left\{S \in \mathcal{S}_+: \text{$S \not= \emptyset$, $\RX_S^\top \RX$ is invertible, and $\Rh_{S^c} > \RX_{S^c}^\top \RX_S (\RX_S^\top \RX_S)^{-1}\Rh_S$}\right\}.
\end{align*}
By construction, for any $S \in \mathcal{S}'_+$, there is a unique $\beta^{\*,S} \in \R_+^d$ such that $\RX \beta^{\*,S} = \Ry$. Besides, the statement of Lemma~\ref{lem:condition_mystere} rewrites as: either $\emptyset \in \mathcal{S}_+$ and $\mathcal{S}'_+$ is empty, or $\emptyset \not\in \mathcal{S}_+$ and $\mathcal{S}'_+$ contains at most one element. Therefore, the proof of Lemma~\ref{lem:condition_mystere} stems from the following result.

\begin{lemma}[Weighted Lasso in the standard noise regime]\label{lem:pf-std}
  Let Assumptions~\ref{ass:exist_nonneg_interp} and~\ref{ass:no_degenerate_coordinate} hold.
  \begin{enumerate}
    \item If $\emptyset \in \mathcal{S}_+$, then $\beta=0$ is the unique solution to~\eqref{eq:WL}, for any $\Rdelta>0$.
    \item Assume that the set $\mathcal{S}'_+$ is nonempty. Define
    \begin{equation*}
      \Rdelta_- := \sup\{\Rdelta>0: \forall S \in \mathcal{S}'_+, \forall k \in S, \beta^{\*,S}_k > \Rdelta [(2\RX_S^\top \RX_S)^{-1} \Rh_S]_k\} > 0,
    \end{equation*}
    and assume that $\Rdelta < \Rdelta_-$.
    \begin{enumerate}
      \item For any $S \in \mathcal{S}'_+$, the pair 
      \begin{equation*}
        \beta^S = [\beta^{\*,S} - \Rdelta(2\RX_S^\top \RX_S)^{-1} \Rh_S, 0_{S^c}]^\top, \quad \mu^S = [0_S, \Rdelta(\Rh_{S^c} - \RX_{S^c}^\top \RX_S (\RX_S^\top \RX_S)^{-1}\Rh_S)]^\top,
      \end{equation*}
      satisfies the condition~\eqref{eq:KKT}; in particular, $\beta^S$ is a solution to~\eqref{eq:WL} with support $S$.
      \item All solutions to~\eqref{eq:WL} coincide.
    \end{enumerate}
  \end{enumerate}
\end{lemma}
\begin{proof}
  We first assume that $\emptyset \in \mathcal{S}_+$. Then $\Ry=0$ and, for any $\Rdelta > 0$, $\RL_{\Rdelta}(\beta) = \|\RX \beta\|_2^2 + \Rdelta \langle \Rh, \beta \rangle$. It is obvious that $\beta=0$ is a global minimiser of $\RL_{\Rdelta}$, and by Assumption~\ref{ass:no_degenerate_coordinate}, it is the only one.
  
  We now assume that the set $\mathcal{S}'_+$ is nonempty, fix $\Rdelta \in (0,\Rdelta_-)$, let $S \in \mathcal{S}'_+$ and define $(\beta^S,\mu^S)$ accordingly. It is simple linear algebra to show that this pair satisfies the condition~\eqref{eq:KKT}; besides, it follows from the condition that $\Rdelta \in (0,\Rdelta_-)$ and the definition of $\mathcal{S}'_+$ that $\supp(\beta^S)=S$ and $\mu^S_{S^c}>0$. We therefore deduce that $\beta^S$ solves~\eqref{eq:WL}. 
  
  It remains to prove that all solutions to~\eqref{eq:WL} coincide. To proceed, we let $\beta' \in \R_+^d$ be another minimiser of $\RL_{\Rdelta}$. Using the condition~\eqref{eq:KKT}, which yields $\nabla \RL_0(\beta^S)+\Rdelta \Rh=\mu^S$ and $\langle \mu^S, \beta^S\rangle=0$, we get
  \begin{align*}
    \RL_0(\beta') - \RL_0(\beta^S) &= \RL_{\Rdelta}(\beta') - \Rdelta \langle \Rh, \beta' \rangle - \RL_{\Rdelta}(\beta^S) + \Rdelta \langle \Rh, \beta^S \rangle\\
    &= \Rdelta \langle \Rh, \beta^S-\beta' \rangle\\
    &= \langle \mu^S - \nabla \RL_0(\beta^S), \beta^S-\beta' \rangle\\
    &= -\langle \mu^S, \beta' \rangle + \langle \nabla \RL_0(\beta^S), \beta'-\beta^S \rangle.
  \end{align*}
  Since $\RL_0$ is convex, $\RL_0(\beta') - \RL_0(\beta^S) \geqslant \langle \nabla \RL_0(\beta^S), \beta'-\beta^S \rangle$ and therefore $\langle \mu^S, \beta'\rangle \leqslant 0$. Since both $\mu^S$ and $\beta^S$ are nonnegative, $\langle \mu^S, \beta'\rangle = 0$, and since $\mu^S_{S^c}>0$, we deduce that $S':=\supp(\beta') \subset S$. Finally, we let $\mu' \in \R_+^d$ be such that $(\beta',\mu')$ satisfies the condition~\eqref{eq:KKT}. On the one hand, the same convexity argument for $\RL_0$ as above yields $\langle \mu', \beta^S\rangle = 0$, which implies that $\mu'_S=0$. On the other hand, we deduce from the condition~\eqref{eq:KKT} for both $(\beta^S,\mu^S)$ and $(\beta',\mu')$ that
  \begin{equation*}
    2\RX^\top \RX(\beta^S-\beta') = \mu^S-\mu',
  \end{equation*}
  which entails $\beta_S^S-\beta_S' = -(2\RX^\top_S \RX_S)^{-1}(\mu_S^S-\mu'_S) = 0$ and completes the proof.
\end{proof}

Lemma~\ref{lem:pf-std} implies in particular that $\mathcal{S}'_+$ contains at most one element, and that if $\emptyset \in \mathcal{S}_+$ then $\mathcal{S}'_+$ is empty. This proves Lemma~\ref{lem:condition_mystere}. We now let Assumption~\ref{ass:mystere} hold, which defines $\beta^\*, \Sstar$, and makes the statement of Theorem~\ref{thm:Lasso} in the case where $\Sstar=\emptyset$ a straightforward consequence of the first part of Lemma~\ref{lem:pf-std}. Likewise, in the case where $\Sstar\not=\emptyset$, the statement of Theorem~\ref{thm:Lasso} in the standard noise regime immediately follows from Lemma~\ref{lem:pf-std}. To complete the proof of the statement in the large noise regime, we resort to the following lemma.

\begin{lemma}[Weighted Lasso in the large noise regime]\label{lem:pf-lge}
  Under the assumptions of Theorem~\ref{thm:Lasso}, let $\Sstar \not= \emptyset$ and assume that $\Rdelta>\Rdelta_+$. Then $\betaL := 0$, $\muL := \Rdelta \Rh - 2 \RX^\top \Ry>0$ is the unique pair of nonnegative vectors which satisfies the condition~\eqref{eq:KKT}.
\end{lemma}
\begin{proof}
  It is clear that $(\betaL,\muL)$ satisfies the condition~\eqref{eq:KKT}. For uniqueness, let $(\beta',\mu')$ be another solution. Then by the same convexity argument as in the proof of Lemma~\ref{lem:pf-std}, we get $\langle \muL, \beta'\rangle=0$, which yields $\beta'=0$ and then $\mu'=\muL$.
\end{proof}

\section{Proof of Theorems~\ref{thm:convergence_large_noise} and \ref{thm:standard_noise}}
\label{app:theorems}

Recall the following facts that are developed in Section~\ref{subsec:dual_processx} of the main text. For $t \in [0,\tau_\infty)$, we introduced a dual variable $\vv(t) \in \R^n$ associated to $\beta(t)$ such that $\beta(t) = \beta(t,\vv(t))= \exp(\RX^\top \vv(t) + \Uz - c t) \in \R^d$, with $\Uz \in \ker\RX$ and $c = \Rdelta \Rh - 2 \RX^\top \Ry \in \R^d$. Denoting $F(t,\vv):= \|\beta(t,\vv)\|_1$, we see that $(\vv(t))_{0 \leqslant t < \tau_\infty}$ follows a nonautonomous overdamped Langevin dynamics with respect to the function $\vv \mapsto F(t,\vv)$~\eqref{eq:SDE_v}, and we also introduce the related gradient flow $(\ww(t))_{0 \leqslant t < \tilde{\tau}_\infty}$~\eqref{eq:GF_w}, which we recall here:
\begin{equation*}
\dd \vv(t) = - 2 \nabla_v F(t,\vv(t)) \dd t + 2 \sqrt{\Rdelta} \dd B_t, \qquad \dd \ww(t) = -2 \nabla_v F(t,\ww(t)) \dd t, \quad \ww(0) = \vv(0).
\end{equation*}
The following statement, the verification of which is straightforward, shall play an important role in the study of~\eqref{eq:SDE_v} and~\eqref{eq:GF_w}.
\begin{lemma}\label{lem:Fcvx}
  For any $t \geqslant 0$, the function $\vv \mapsto F(t,\vv)$ is convex on $\R^n$.
\end{lemma}

\subsection{The large noise regime: Theorem~\ref{thm:convergence_large_noise} }
\label{app:large_noise_regime}

Here we place ourselves in the setting of Theorem~\ref{thm:convergence_large_noise}. Then the vector $c$ coincides with the vector $\muL>0$ from Theorem~\ref{thm:Lasso}. Let us denote $\cmin>0$ its smallest entry. As depicted in Section~\ref{subsec:large_noise}, we divide the proof in two steps.
\begin{enumerate}[label=(\roman*)]
\item First we show that the gradient flow $(\ww(t))_{0 \leqslant t < \tilde{\tau}_\infty}$ is bounded, which implies that $\tilde{\tau}_\infty=+\infty$. This is presented in Proposition~\ref{prop:GF} of Section~\ref{subsub:GF_bound}.
\item Then we show that $(\vv(t))_{0 \leqslant t < \tau_\infty}$ remains in its vicinity with Brownian fluctuations of order $\sqrt{t}$, which implies that $\tau_\infty=+\infty$. This is presented in Proposition~\ref{prop:SDE-GF} of Section~\ref{subsub:SDE-GF_bound}.
\end{enumerate}
We finally prove Theorem \ref{thm:convergence_large_noise} in Section~\ref{subsub:final_bound}.

\subsubsection{Bounding the gradient flow}
\label{subsub:GF_bound}

\begin{proposition}
\label{prop:GF}
Let $(\ww(t))_{0 \leqslant t < \tilde{\tau}_\infty}$ be the solution of the gradient flow~\eqref{eq:GF_w}. For all $t \in [0,\tilde{\tau}_\infty)$,
\begin{align*}
\forall w \in \R^n, \qquad \|\ww(t) - w\|_2 \leqslant  \|w(0) - w\|_2 + 2\|\e^{\RX^\top w + \Uz}\|_2 \frac{\|\RX\|}{\cmin}.
\end{align*}
In particular, $\tilde{\tau}_\infty=+\infty$.
\end{proposition}
\begin{proof}
Thanks to Lemma~\ref{lem:Fcvx}, we have, for all $w \in \R^n$, for all $t \in [0,\tilde{\tau}_\infty)$, 
\begin{align*}
\frac{\dd}{\dd t} \frac{1}{2} \|\ww(t) - w\|_2^2 &= \left\langle \frac{\dd}{\dd t}\ww(t), \ww(t) - w\right\rangle \\ 
&= -2\langle \nabla_v F(t,\ww(t)), \ww(t) - w \rangle\\ 
&= -2\langle \nabla_v F(t,\ww(t))-\nabla_v F(t,w), \ww(t) - w \rangle - 2\langle \nabla_v F(t,w), \ww(t) - w \rangle \\ 
&\leqslant -2\langle \nabla_v F(t,w), \ww(t) - w \rangle \\ 
&\leqslant  2\|\nabla_v F(t,w)\|_2\|\ww(t) - w \|_2 \\ 
&=  2\|\RX \e^{\RX^\top w + \Uz - ct}\|_2\|\ww(t) - w \|_2\\ 
&\leqslant C \ee^{-\cmin t}\|\ww(t) - w \|_2,
\end{align*}
with $C = 2\|\RX\| \|\e^{\RX^\top w + \Uz}\|_2$. Then, denote $f(t) := \|\ww(t) - w \|_2^2$, we have, 
\begin{align*}
\frac{1}{2} f'(t) &\leqslant C \e^{- \cmin t} \sqrt{f(t)}, 
\end{align*}
so that, for any $\varepsilon >0$, $g^\varepsilon(t) := \sqrt{f(t) + \varepsilon}$ satisfies
\begin{align*}
[g^\varepsilon(t)]' = \frac{f'(t)}{2\sqrt{ f(t) + \varepsilon}} &\leqslant C  \e^{- \cmin t} \frac{\sqrt{f(t)}}{\sqrt{ f(t) + \varepsilon}} \leqslant C  \e^{- \cmin t},
\end{align*}
and by direct integration,
\begin{align*}
g^\varepsilon(t) &\leqslant g^\varepsilon(0) + C \int_0^t \e^{- \cmin s} \dd s = g^\varepsilon(0) + \frac{C}{\cmin}  (1 - \e^{- \cmin t}) \leqslant \sqrt{f(0) + \varepsilon} + \frac{C}{\cmin} .
\end{align*}
By letting $\varepsilon \to 0$, we have $\sqrt{f(t)} \leqslant \sqrt{f(0)}  + \frac{C}{\cmin}$, which proves the proposition.
\end{proof}

\subsubsection{Control of the stochasticity: comparison between gradient and stochastic gradient flow}
\label{subsub:SDE-GF_bound}

The aim of this section is to bound the difference between the SDE solution of Eq.~\eqref{eq:SDE_v} and the gradient flow solution of Eq.~\eqref{eq:GF_w}. By Itô calculus, we have for $t \in [0, \tau_\infty)$,
\begin{align*}
\dd \|\vv(t) - \ww(t)\|_2^2 &= 2 \langle \vv(t) - \ww(t), \dd \vv(t) - \dd \ww(t) \rangle + 4n\Rdelta \dd t \\
&= -4\left( \langle \vv(t) - \ww(t), \nabla_v F(t,\vv(t)) - \nabla_v F(t,\ww(t)) \rangle - n\Rdelta\right)\dd t\\
&\quad + 4 \sqrt{\Rdelta} \langle \vv(t) - \ww(t), \dd B(t) \rangle.
\end{align*}
We begin by rewriting the previous equation with a one dimensional Brownian motion.
\begin{lemma}\label{lem:W}
There exists a one dimensional Brownian motion $(W(t))_{t \geqslant 0}$ such that, almost surely, for $t \in [0,\tau_\infty)$,
\begin{align*}
  \dd \|\vv(t) - \ww(t)\|_2^2 &= -4\left(\langle \vv(t) - \ww(t), \nabla_v F(t,\vv(t)) - \nabla_v F(t,\ww(t)) \rangle - n\Rdelta \right)\dd t\\
  &\quad + 4 \sqrt{\Rdelta} \|\vv(t) - \ww(t)\|_2 \dd W(t).
\end{align*}
\end{lemma}
\begin{proof}
From regularity theory of elliptic equations, for any $t>0$ the law of $\vv(t)$ on the event $\{t<\tau_\infty\}$ is absolutely continuous with respect to the Lebesgue measure and hence $\Pr(t < \tau_\infty, \vv(t) = \ww(t) ) = 0$. As a consequence, $\E[\int_0^t \ind{s < \tau_\infty, \vv(s)=\ww(s)} \dd s] = \int_0^t \Pr(s < \tau_\infty, \vv(s)=\ww(s))\dd s = 0$, hence $\int_0^t \ind{s < \tau_\infty, \vv(s)=\ww(s)} \dd s = 0$, almost surely. Therefore, by Levy's characterisation, the local martingale $(W(t))_{t \geqslant 0}$ defined by
\begin{equation*}
  \forall t \geqslant 0, \qquad W(t) = \int_0^t \left(\ind{s<\tau_\infty} \frac{\langle \vv(s) - \ww(s), \dd B(s)\rangle}{\|\vv(s) - \ww(s)\|_2} + \ind{s\geqslant \tau_\infty}\dd B^1(s)\right),
\end{equation*}
where $(B^1(t))_{t \geq 0}$ is the first coordinate of $(B(t))_{t \geq 0}$, is a Brownian motion. This completes the proof.
\end{proof}
From this, using Lemma~\ref{lem:Fcvx}, we have
\begin{align*}
\dd \|\vv(t) - \ww(t)\|_2^2 \leqslant 4n\Rdelta \dd t + 4 \sqrt{\Rdelta} \|\vv(t) - \ww(t)\|_2 \dd W(t).
\end{align*}
Then, define $Y(t):=\|\vv(t) - \ww(t)\|_2^2$, we have $Y(0) = 0$, and for $t \in [0,\tau_\infty)$,
\begin{align*}
\dd Y(t) \leqslant 4n\Rdelta\dd t + 4 \sqrt{\Rdelta Y(t)}\, \dd W(t).
\end{align*}
To control $Y(t)$, it is interesting to introduce the process that saturates the constraint in the inequality above. It turns out that we can provide an exact representation of the distribution of this process, which is the squared norm of an $n$-dimensional Brownian motion.
\begin{proposition}\label{prop:SDE-GF}
  The stochastic differential equation
  \begin{equation}\label{eq:RZd}
    \dd \RZd(t) = 4n\Rdelta\dd t + 4 \sqrt{\Rdelta \RZd(t)} \dd W(t), \qquad \RZd(0)=0,
  \end{equation}
  has a unique strong solution $(\RZd(t))_{t \geqslant 0}$, which is defined globally in time. Besides:
  \begin{enumerate}[label=(\roman*)]
    \item the process $(\RZd(t))_{t \geqslant 0}$ has the same law as $(4\Rdelta \|B(t)\|_2^2)_{t \geqslant 0}$;
    \item almost surely, for any $t \in [0,\tau_\infty)$, $Y(t) \leqslant \RZd(t)$.
  \end{enumerate}
\end{proposition}
\begin{proof}
  First, we define $\RZdp(t) := 4\Rdelta \|B(t)\|_2^2$ and deduce from Itô's formula that
  \begin{equation*}
    \dd \RZdp(t) = 4n\Rdelta\dd t + 8 \Rdelta \langle B(t), \dd B(t)\rangle.
  \end{equation*}
  By the same arguments as in the proof of Lemma~\ref{lem:W}, there exists a one-dimensional Brownian motion $(\tilde{W}(t))_{t \geqslant 0}$ such that $\langle B(t), \dd B(t)\rangle = \|B(t)\|_2 \dd \tilde{W}(t)$, which yields 
  \begin{equation*}
    8 \Rdelta \langle B(t), \dd B(t)\rangle = 4 \sqrt{\Rdelta \RZdp(t)} \dd \tilde{W}(t)
  \end{equation*}
  and shows that $(\RZdp(t),\tilde{W}(t))_{t \geqslant 0}$ is a weak solution to the SDE~\eqref{eq:RZd}. On the other hand, by Theorem 3.5-(ii) of \cite[Chapter~IX]{revuz2013continuous}, pathwise uniqueness is known to hold for this SDE. Therefore, by the Yamada--Watanabe theorem~\citep[Chapter 5, Corollary 3.23]{karatzas2012brownian}, strong existence also holds. Besides, by uniqueness in law~\citep[Chapter 5, Proposition 3.20]{karatzas2012brownian}, the strong solution $(\RZd(t))_{t \geqslant 0}$ driven by $(W(t))_{t \geqslant 0}$ has the same law as $(\RZdp(t))_{t \geqslant 0}$. 
  
  To prove the last statement of the proposition, we follow the lines of~\cite[Theorem~3.7, Chapter~IX]{revuz2013continuous}. First, for any $M \geqslant 0$, we set $\tau_M := \inf\{t \geqslant 0: \|\beta(t)\|_2 \geqslant M\}$, so that $\tau_\infty = \limsup_{M \to +\infty} \tau_M$. We may then write, for any $t \geqslant 0$,
  \begin{equation*}
    (Y(t \wedge \tau_M)-\RZd(t \wedge \tau_M))_+ \leqslant 4\sqrt{\Rdelta}\int_0^{t \wedge \tau_M} \ind{Y(s)\geqslant \RZd(s)}\left(\sqrt{Y(s)}-\sqrt{\RZd(s)}\right)\dd W(s),
  \end{equation*}
  which then implies that $\mathbb{E}[(Y(t \wedge \tau_M)-\RZd(t \wedge \tau_M))_+]=0$ and therefore, by continuity of the trajectories of $Y(t \wedge \tau_M)-\RZd(t \wedge \tau_M)$, we have almost surely
  \begin{equation*}
    \forall t \geqslant 0, \quad \forall M \geqslant 0, \qquad Y(t \wedge \tau_M) \leqslant \RZd(t \wedge \tau_M).
  \end{equation*} 
  The final claim easily follows.
\end{proof}

%

\subsubsection{Convergence of the initial stochastic flow to the origin}
\label{subsub:final_bound}

We are now in place to give a bound on the iterates $\beta(t,\vv(t))$ for all $t \geqslant 0$.

\begin{theorem}
\label{thm:app_almost_sure_convergence}
Let $(\RZd(t))_{t \geqslant 0}$ be defined by Proposition~\ref{prop:SDE-GF}. Almost surely,
\begin{equation*}
  \forall t \in [0,\tau_\infty), \qquad \beta(t) \leqslant \e^{\|\RX^\top\|(\sqrt{\RZd(t)} + 2\|\RX\|\|\beta(0)\|_2/\cmin)} \beta(0) \odot \e^{-ct}.
\end{equation*}
In consequence, $\tau_\infty=+\infty$ and $\lim_{t \to +\infty} \beta(t) = 0$ almost surely.
\end{theorem}
\begin{proof}
Thanks to Propositions~\ref{prop:GF} and~\ref{prop:SDE-GF}, we have the following almost sure inequality
\begin{align*}
\|\vv(t) - \vv(0)\|_2 &\leqslant \|\vv(t) - \ww(t)\|_2 + \|\ww(t) - \ww(0)\|_2\\
&\leqslant \sqrt{\RZd(t)} + 2\|\beta_0\|_2\frac{\|\RX\|}{\cmin}.
\end{align*}
Transferring this estimation to the iterates of the initial flow, we get,
\begin{align*}
\beta(t) &= \exp\left(\RX^\top \vv(t) + \Uz - ct\right)\\
&= \beta(0) \odot \exp\left(\RX^\top (\vv(t)-\vv(0)) - ct\right)\\
&\leqslant \e^{\|\RX^\top\| \|\vv(t) - \vv(0)\|_2} \beta(0) \odot \ee^{-ct}\\
&\leqslant \e^{\|\RX^\top\| (\sqrt{\RZd(t)} + 2\|\beta_0\|_2\|\RX\|/\cmin)} \beta(0) \odot \ee^{-ct},
\end{align*}
which proves the first inequality in the theorem and implies that $\tau_\infty=+\infty$, almost surely. To prove that $\beta(t)$ goes to $0$, we first note that, by the law of the iterated logarithm~\cite[Corollary~1.12, Chapter~II]{revuz2013continuous},
\begin{equation*}
  \limsup_{t \to +\infty} \frac{\|B(t)\|^2_2}{2 t \log \log(t)} \leqslant n, \qquad \text{almost surely,}
\end{equation*}
which by Proposition~\ref{prop:SDE-GF} then implies that
\begin{equation*}
  \limsup_{t \to +\infty} \|\RX^\top\|\sqrt{\RZd(t)}\iind{} - c t = -\infty, \qquad \text{almost surely,}
\end{equation*}
and completes the proof.
\end{proof}

\subsection{The standard noise regime: Theorem \ref{thm:standard_noise}}
\label{sub:standard_noise}

\subsubsection{Introduction of the dual variables}

We now work in the setting of Theorem~\ref{thm:standard_noise}, where in particular $S^\*\not=\emptyset$ and $\Rdelta < \Rdelta_-$. We recall the explicit form of $\betaL$ and its conic dual variable $\muL$ as derived in Theorem~\ref{thm:Lasso},
\begin{align*}
\betaL = [\underbrace{\beta^\*_{S^\*} - \Rdelta(2 \RX_{S^\*}^\top \RX_{S^\*})^{-1} \Rh_{S^\*}}_{\betaL_{S^\*}>0} ,\, 0_{{S^\*}^c} ], \qquad
\muL = [0_{S^\*},\ \underbrace{\Rdelta (\Rh_{{S^\*}^c} - \RX^\top_{{S^\*}^c} \RX_{S^\*} (\RX_{S^\*}^\top \RX_{S^\*})^{-1} \Rh_{S^\*})}_{\muL_{{S^\*}^c}>0}].
\end{align*} 
As said in Section~\ref{subsec:small_noise}, the conceptual crux is to see the problem like a maximum margin selection of a linearly separable classification problem \textit{on the data points given by the coordinates of the $(\Rx_i)_{1 \leqslant i \leqslant n}$}. Let us recall the definition of this transposed dataset more precisely: define the $d$ data inputs $(\Rz_k)_{1 \leqslant k\leqslant d} \in (\R^{n})^d$ such that $\RZ = \RX^\top$, i.e for all $k \in \[1,d\]$, $i \in \[1,n\]$ $[\Rz_k]_i = [\Rx_i]_k$. Let us define, the following dual variables of $\R^n$:
\begin{align*}
\vL    &:= \RX_{S^\*} (\RX^\top_{S^\*} \RX_{S^\*})^{-1} \log(\betaL_{S^\*}),  \\
\vinf  &:= -2 \RX\betaL = -2 \RX_{S^\*} \betaL_{S^\*} .
\end{align*} 
The two variables have the following property that will be crucial to derive the behaviour of the dual process \eqref{eq:SDE_v}, and easily follows from the definitions of $\vL$, $\vinf$ together with the condition~\eqref{eq:KKT}. We recall here that $c=\Rdelta \Rh - 2 \RX^\top \Ry$. 
\begin{lemma}\label{lem:vLvinf}
We have the following properties on $(\vL, \vinf)$:
\begin{enumerate}[label=(\roman*)]
\item $ \vinf = -2 \RX_{S^\*} \e^{\RX_{S^\*}^\top \vL}$, i.e translated in terms of the $\Rz$'s variables as  $  \vinf = -2\sum_{k \in S^\*} \Rz_k\e^{\langle \Rz_k, \vL \rangle }$. 
\item $ - \RX^\top \vinf + c = \muL $. This translates in terms of $\Rz$'s variables as for $k \in S^\*$,  $-\langle \Rz_k, \vinf\rangle + c_k = 0$ and for $k \in {S^\*}^c$, $- \langle \Rz_k, \vinf\rangle + c_k =  \muL_k > 0$.
\end{enumerate}
\end{lemma}
\subsubsection{The residual process}
We have introduced such vectors for a concrete purpose. Indeed, we are going to show that at first order, the process $(\vv(t))_{0 \leqslant t < \tau_\infty}$ defined in Eq.~\eqref{eq:SDE_v} will approximately diverge as $t$ along the ray $\{\vinf t + \vL , t \geqslant 0\}$. Define, for $t \in [0,\tau_\infty)$, what can be called the residual process,
$$r(t):= \vv(t) - \vinf t - \vL + \RX_{S^\*} (\RX_{S^\*}^\top \RX_{S^\*})^{-1} [\Uz]_{S^\*} , $$
where $[\Uz]_{S^\*} \in \R^s$ is the restriction of $\Uz \in \R^d$ on the support $S^\*$. We also define $\Uz' := \RX^\top_{{S^\*}^c}\RX_{S^\*} (\RX_{S^\*}^\top \RX_{S^\*})^{-1} [\Uz]_{S^\*} \in \R^{d-s}$, and deduce from Itô calculus and Lemma~\ref{lem:vLvinf} that
\begin{align*}
\dd \|r(t)\|_2^2 &= 2 \langle \dd r(t), r(t) \rangle + 4n\Rdelta \dd t  \\
&= -4\langle \RX \e^{\RX^\top \vv(t) + \Uz - c t} - \vinf, r(t) \rangle  \dd t + 4n\Rdelta \dd t + 4  \sqrt{\Rdelta} \langle r(t), \dd B(t)\rangle\\
&= - 4\langle  \RX \e^{\RX^\top r(t) + \RX^\top \vL + \Uz - \Uz' + (\RX^\top\vinf- c) t} + \vinf,\, r(t) \rangle  \dd t + 4n\Rdelta \dd t + 4  \sqrt{\Rdelta} \langle r(t), \dd B(t)\rangle   \\
&= - 4 \langle \RX \left(\e^{\RX^\top r(t) + \RX^\top \vL + \Uz - \Uz'  - \muL t} - \betaL\right),\, r(t) \rangle  \dd t + 4n\Rdelta \dd t + 4  \sqrt{\Rdelta} \langle r(t), \dd B(t)\rangle   \\
&= - 4 \sum_{k \in S^\*} \e^{ \langle \Rz_k, \vL \rangle} \left(\e^{\langle \Rz_k, r(t) \rangle} - 1 \right) \langle \Rz_k, r(t) \rangle \dd t\\
&\quad - 4 \sum_{k \in {S^\*}^c} \e^{\langle \Rz_k, \vL \rangle + [\Uz]_k - [\Uz']_k -\muL_k t } \e^{\langle \Rz_k, r(t)\rangle }\langle \Rz_k, r(t) \rangle \dd t + 4n\Rdelta \dd t + 4  \sqrt{\Rdelta} \langle r(t), \dd B(t)\rangle.
\end{align*}
By the same argument as in Lemma~\ref{lem:W}, there exists a one dimensional Brownian motion $(W(t))_{t \geqslant 0}$ such that for all $t \in [0,\tau_\infty)$, 
\begin{align*}
\dd \|r(t)\|_2^2 &= - 4 \sum_{k \in S^\*} \betaL_k \left(\e^{\langle \Rz_k, r(t) \rangle} - 1 \right) \langle \Rz_k, r(t) \rangle \dd t\\
& \quad - 4 \sum_{k \in {S^\*}^c} \e^{\langle \Rz_k, \vL \rangle + [\Uz]_k - [\Uz']_k -\muL_k t } \e^{\langle \Rz_k, r(t)\rangle }\langle \Rz_k, r(t) \rangle \dd t + 4n\Rdelta \dd t + 4  \sqrt{\Rdelta} \|r(t)\| \dd W(t).
\end{align*}
Using that for any $x \in \R$, $-x \e^x \leqslant |x|$, defining $\mumin = \min_{k \in {S^\*}^c} \muL_k$, and finally noting the sum $b = \sum_{k \in {S^\*}^c} \e^{\langle \Rz_k, \vL \rangle + [\Uz]_k - [\Uz']_k} \|\Rz_k\|_2$, we get that the second term is upper bounded by $ b \e^{-\mumin t}\|r(t)\|_2$.

The first term is a bit more involved and crucially rests on the strong convexity on the support given by the invertibility of $\RX^\top_{S^\*}\RX_{S^\*}$. Indeed, first note that for any $x \in \R$, $x (\e^x - 1) \geqslant x^2/(1 + |x|)$. Moreover if we denote $a = \min_{k \in S^\*}{\betaL_k}$, $\Omega = \sup_{k \in S^\*} \|z_k\|_2$, and let $\rho_{S^\*}>0$ such that $\RX^\top_{S^\*}\RX_{S^\*} \geqslant \rho_{S^\*} I_{S^\*}$, 
\begin{align*}
- 4 \sum_{k \in S^\*} \betaL_k \left(\e^{\langle \Rz_k, r(t) \rangle} - 1 \right) \langle \Rz_k, r(t) \rangle &\leqslant -4 a \sum_{k \in S^\*} \frac{\langle \Rz_k, r(t) \rangle^2}{1 + |\langle \Rz_k, r(t) \rangle|} \\ 
&\leqslant  -4 a \sum_{k \in S^\*} \frac{\langle \Rz_k, r(t) \rangle^2}{1 + \Omega \|r(t)\|_2} \\
&\leqslant  -\frac{4 a}{1 + \Omega \|r(t)\|_2} \sum_{k \in S^\*} \langle \Rz_k, r(t) \rangle^2 \\
&\leqslant  -4 a \rho_{S^\*} \frac{ \|r(t)\|_2^2}{1 + \Omega \|r(t)\|_2}.
\end{align*}
Finally, 
\begin{equation}
\label{eq:r_t_inequality}
\dd \|r(t)\|_2^2 \leqslant \left( 4n\Rdelta -4 a \rho_{S^\*} \frac{ \|r(t)\|_2^2}{1 + \Omega \|r(t)\|_2} + 4b\e^{-\mumin t}\|r(t)\|_2\right) \dd t + 4 \sqrt{\Rdelta} \|r(t)\|_2 \dd W(t).
\end{equation}
\subsubsection{The comparison process \texorpdfstring{$(\Xid(t))_{t \geqslant 0}$}{} and the first part of the theorem}
Everything is now in order to apply the same SDE comparison argument we detailed in the nonstrongly convex case in Subsection~\ref{subsub:SDE-GF_bound}. 

\begin{proposition}\label{prop:Xid}
  Assume that $n \geqslant 2$ and $r(0) \neq 0$. The stochastic differential equation
  \begin{equation}\label{eq:sde-xi}
    \dd \Xid(t) = \left( 4n\Rdelta -4 a \rho_{S^\*} \frac{ \Xid(t)}{1 + \Omega \sqrt{\Xid(t)}} + 4b\e^{-\mumin t} \sqrt{\Xid(t)}\right) \dd t + 4 \sqrt{\Rdelta \Xid(t)} \dd W(t),
  \end{equation}
  initialised at $\Xid(0) = \|r(0)\|_2^2$, has a unique strong solution $(\Xid(t))_{t \geqslant 0}$, which is defined globally in time. Besides:
  \begin{enumerate}[label=(\roman*)]
    \item the process $(\Xid(t))_{t \geqslant 0}$ has the same law as $(\|\RRd(t)\|_2^2)_{t \geqslant 0}$, where $(\RRd(t))_{t \geqslant 0}$ is the unique strong solution to the $n$-dimensional SDE
    \begin{equation}\label{eq:sde-RRd}
      \dd \RRd(t) = \left(-2 a \rho_{S^\*} \frac{\RRd(t)}{1+\Omega \|\RRd(t)\|_2} + 2b \ee^{-\mumin t} \frac{\RRd(t)}{\|\RRd(t)\|_2}\right)\dd t + 2\sqrt{\Rdelta}\dd B(t),
    \end{equation}
    initialised at $\RRd(0)=r(0)$;
    \item almost surely, for any $t \in [0,\tau_\infty)$, $\|r(t)\|_2^2 \leqslant \Xid(t)$.
  \end{enumerate}
\end{proposition}
\begin{proof}
  The well-posedness of~\eqref{eq:sde-RRd} follows from~\citet{veretennikov1981strong}, 
   and Itô's formula together with Lévy's characterisation show that the process $(\|\RRd(t)\|_2^2)_{t \geqslant 0}$ is a weak solution to~\eqref{eq:sde-xi}. However, in contrast with the proof of Proposition~\ref{prop:SDE-GF}, the presence of the square root in the drift of the SDE~\eqref{eq:sde-xi} prevents us from using standard results to claim pathwise uniqueness. To recover this property, we note that since the drift remains Lipschitz continuous, uniformly in time, on all sets of the form $[\varepsilon, +\infty)$, $\varepsilon>0$, any two strong solutions to~\eqref{eq:sde-xi} coincide, and have the same law as the process $\|\RRd(t)\|_2$, until they hit $0$. As a consequence, to obtain pathwise uniqueness for~\eqref{eq:sde-xi} it suffices to check that, almost surely, $\RRd(t)$ never hits $0$. But since the drift in~\eqref{eq:sde-RRd} is bounded, uniformly in time, the Girsanov theorem shows that for any $T>0$, the laws of $(\RRd(t))_{t \in [0,T]}$ and $(r(0)+2\sqrt{\Rdelta}B(t))_{t \in [0,T]}$ are equivalent, and therefore
  \begin{equation*}
    \Pr\left(\exists t \in [0,T]: \RRd(t)=0\right) = \Pr\left(\exists t \in [0,T]: r(0)+2\sqrt{\Rdelta}B(t)=0\right) = 0,
  \end{equation*}
  where the second equality is well-known in dimension $n \geqslant 2$.
  This implies that, almost surely, $\RRd(t)\not=0$, for any $t \geqslant 0$, and therefore completes the proof of pathwise uniqueness for~\eqref{eq:sde-xi}. 
  
  We now detail the comparison between $\|r(t)\|_2^2$ and $\Xid(t)$. For any $M > M_0 := 1 \vee \|r(0)\|_2^{-2}$, we set $\tau'_M := \inf\{t \in [0,\tau_\infty) : \|r(t)\|^2_2 \wedge \Xid(t) \leqslant 1/M \text{ or } \|\beta(t)\|_2 \geqslant M\}$, so that $\limsup_{M \to +\infty} \tau'_M = \tau_0 \wedge \tau_\infty$, with $\tau_0 :=  \inf\{t \in [0,\tau_\infty): \|r(t)\|_2 = 0\}$ (we recall that, almost surely, $\Xid(t)$ never hits $0$). With similar arguments to~\cite[Theorem~3.7, Chapter~IX]{revuz2013continuous}, we may write
  \begin{align*}
    \mathbb{E}\left[\left(\|r(t \wedge \tau'_M)\|^2_2 - \Xid(t \wedge \tau'_M)\right)_+\right] &\leqslant \mathbb{E}\left[\int_0^{t \wedge \tau'_M} \ind{\|r(s)\|^2_2 \geqslant \Xid(s)}\left|g(s,\|r(s)\|^2_2)-g(s,\Xid(s))\right|\dd s\right],
  \end{align*}
  where $g(t,\xi) = 4n\Rdelta -4 a \rho_{S^\*} \xi/(1 + \Omega \sqrt{\xi}) + 4b\e^{-\mumin t} \sqrt{\xi}$ is the drift of~\eqref{eq:sde-xi}. Denoting by $C_M$ the Lipschitz constant of $g(s,\cdot)$ on $[1/M, +\infty)$, which is uniform in $s$, we get
  \begin{align*}
    \mathbb{E}\left[\left(\|r(t \wedge \tau'_M)\|^2_2 - \Xid(t \wedge \tau'_M)\right)_+\right] &\leqslant C_M \mathbb{E}\left[\int_0^{t \wedge \tau'_M} \ind{\|r(s)\|^2_2 \geqslant \Xid(s)}\left|\|r(s)\|^2_2-\Xid(s)\right|\dd s\right]\\
    &= C_M \mathbb{E}\left[\int_0^{t \wedge \tau'_M} \left(\|r(s)\|^2_2-\Xid(s)\right)_+\dd s\right]\\
    &\leqslant C_M \int_0^t \mathbb{E}\left[\left(\|r(s \wedge \tau'_M)\|^2_2-\Xid(s \wedge \tau'_M)\right)_+\right]\dd s,
  \end{align*}
  which by Gronwall's lemma and the continuity of the trajectories of $\|r(t \wedge \tau'_M)\|^2_2 - \Xid(t \wedge \tau'_M)$ yields, almost surely,
  \begin{equation*}
    \forall t \geqslant 0, \quad \forall M > M_0, \qquad \|r(t \wedge \tau'_M)\|^2_2 \leqslant \Xid(t \wedge \tau'_M),
  \end{equation*}
  and therefore, almost surely,
  \begin{equation*}
    \forall t \in [0,\tau_\infty \wedge \tau_0), \qquad \|r(t)\|^2_2 \leqslant \Xid(t).
  \end{equation*}
  To complete the proof, let us fix $T \in [0,\tau_\infty)$ and call $z := \inf_{t \in [0,T]} \Xid(t) > 0$. The previous argument shows that $\|r(t)\|^2_2 \leqslant \Xid(t)$ for all $t \leqslant T \wedge \tau_0$, and if $\tau_0 < T$, this inequality remains trivially true as long as $t \leqslant \tau^{(1)}_z := \inf\{t \in [\tau_0, T] : \|r(t)\|_2^2 \geqslant z\}$. But if $\tau^{(1)}_z < T$, the argument above can be repeated to show that the inequality holds up to $\tau^{(1)}_0 := \inf\{t \in [\tau^{(1)}_z, T]: \|r(t)\|_2 = 0\}$, and then, if $\tau^{(1)}_0 < T$, up to $\tau^{(2)}_z := \inf\{t \in [\tau^{(1)}_0, T] : \|r(t)\|_2^2 \geqslant z\}$. Iterating the argument, we thus construct two sequences $\tau^{(l)}_0, \tau^{(l)}_z \leqslant T$ such that $\|r(\tau^{(l)}_0)\|_2 = 0$ if $\tau^{(l)}_0 < T$, and $\|r(\tau^{(l)}_z)\|_2^2 \geqslant z$ if $\tau^{(l)}_z < T$, and such that the inequality $\|r(t)\|^2_2 \leqslant \Xid(t)$ holds on $[0,\tau^{(l)}_0]$. By continuity of the trajectory of $r(t)$ on $[0,T]$, there are only finitely many $\tau^{(l)}_0, \tau^{(l)}_z$ which are strictly below $T$, and therefore the inequality $\|r(t)\|^2_2 \leqslant \Xid(t)$ finally holds on $[0,T]$.
\end{proof}

We are now in place to give a bound on the iterates $\beta(t,\vv(t))$ for all $t \in [0,\tau_\infty)$.

\begin{theorem}
\label{thm:final_bound}
Let $(\Xid(t))_{t \geqslant 0}$ be defined in Proposition~\ref{prop:Xid}. 
\begin{enumerate}[label=(\roman*)]
\item On the support: almost surely, 
\begin{equation*}
\forall t \in [0,\tau_\infty), \qquad \betaL_{S^\*}\e^{-\|\RX_{S^\*}^\top\|\sqrt{\Xid(t)}} \leqslant \beta_{S^\*}(t)   \leqslant \betaL_{S^\*}\e^{\|\RX_{S^\*}^\top\| \sqrt{\Xid(t)}}.
\end{equation*}
\item Outside the support: there exists $C \geqslant 0$, depending on $\RX$, $\Ry$ and $\beta(0)$, such that, almost surely, 
\begin{equation*}
\forall t \in [0,\tau_\infty), \qquad \beta_{{S^\*}^c}(t) \leqslant C \exp\left(\|\RX_{{S^\*}^c}^\top\|\sqrt{\Xid(t)}\iind{{S^\*}^c} - \muL_{{S^\*}^c} t\right).
\end{equation*}
\end{enumerate}
In particular, $\tau_\infty=+\infty$, almost surely.
\end{theorem}
\begin{proof}
We transfer the estimates on $\vv(t)$ to the iterates of the initial flow on the linear predictor $\beta(t)$. Indeed  as $\RX^\top \vv(t) =  \RX^\top r(t) + \RX^\top \vinf t + \RX^\top \vL - [[\Uz]_{S^\*}, \Uz']^\top$, we get for $t \in [0,\tau_\infty)$,
\begin{align*}
\beta(t) &= \exp\left(\RX^\top \vv(t) + \Uz - c t \right) \\
&= \exp\left( \RX^\top r(t) + \RX^\top \vinf t + \RX^\top \vL - [[\Uz]_{S^\*}, \Uz']^\top + \Uz  - c t \right) \\
&= \exp\left( \RX^\top r(t) + \RX^\top \vL + [0_{S^\*}, [\Uz]_{{S^\*}^c}-\Uz']^\top + (\RX^\top\vinf - c)t\right) \\
&= \exp\left( \RX^\top r(t) + \RX^\top \vL + [0_{S^\*}, [\Uz]_{{S^\*}^c}-\Uz']^\top - \muL t\right),
\end{align*}
thanks to Lemma~\ref{lem:vLvinf}. On the support, we have $[\e^{\RX^\top \vL}]_{S^\*} = \betaL_{S^\*}$ and $\muL_{S^\*}=0$. Hence, $\beta_{S^\*}(t) =  \betaL_{S^\*} \odot\exp\left(\RX_{S^\*}^\top r(t) \right)$, and the first part of the Theorem follows from Proposition~\ref{prop:Xid}. The second part of the Theorem follows similarly, with $C = \max_{k \in {S^\*}^c} \exp([\RX^\top \vL]_k + [\Uz]_k-[\Uz']_k)$. Finally, the fact that $\Xid(t)$ does not explode yields $\tau_\infty=+\infty$, which completes the proof.
\end{proof}

\subsubsection{Study of the process \texorpdfstring{$(\RRd(t))_{t \geqslant 0}$}{} and the second part of the theorem}
\label{subsub:R}

This section is dedicated to a detailed study of the long-time behaviour of the process $(\RRd(t))_{t \geqslant 0}$, and therefore of $(\Xid(t))_{t \geqslant 0}$. To proceed, we first rewrite~\eqref{eq:sde-xi} under the form
\begin{equation*}
  \dd \RRd(t) = G(t,\RRd(t))\dd t + 2\sqrt{\Rdelta}\dd B(t), \qquad G(t,r) := -\nabla V(r) + 2b \ee^{-\mumin t} \frac{r}{\|r\|_2},
\end{equation*}
where $V : \R^n \to \R$ is the $C^1$ function with globally bounded gradient defined by
\begin{equation}
\label{eq:potential}
  V(r) := \frac{2a\rho_{S^\*}}{\Omega^2}\left(\Omega \|r\|_2 - \log(1+\Omega \|r\|_2)\right).
\end{equation}
We also introduce the time-homogeneous diffusion process
\begin{equation*}
  \dd \RRdb(t) = -\nabla V(\RRdb(t))\dd t + 2\sqrt{\Rdelta}\dd B(t), \qquad \RRdb(0) = r(0).
\end{equation*}
We first state a trajectorial comparison result between $\RRd(t)$ and $\RRdb(t)$.

\begin{proposition}\label{prop:RRdRRdb}
  Almost surely,
  \begin{equation*}
    \sup_{t \geqslant 0} \|\RRdb(t)-\RRd(t)\|_2 \leqslant \frac{b}{\mumin}, \qquad \text{and} \qquad \lim_{t \to +\infty} \|\RRdb(t)-\RRd(t)\|_2 = 0.
  \end{equation*}
\end{proposition}
\begin{proof}
  \emph{Step~1. Convexity of $V$.} We start the proof by showing that $V$ is convex on $\R^n$ and bounding $\nabla^2 V(r)$ from below. To proceed, we compute
\begin{equation}
\label{eq:hessian_potential}
  \nabla^2 V(r) = \frac{2a\rho_{S^\*}}{(1+\Omega \|r\|_2)^2}\left((1+\Omega \|r\|_2)I_n - \Omega\frac{rr^\top}{\|r\|_2}\right),
\end{equation}
so that, for any $u \in \R^n$,
\begin{align*}
  \langle u, \nabla^2 V(r) u \rangle &= \frac{2a\rho_{S^\*}}{(1+\Omega \|r\|_2)^2}\left((1+\Omega \|r\|_2)\|u\|_2^2 - \Omega\frac{\langle r, u\rangle^2}{\|r\|_2}\right)\\
  &\geqslant \frac{2a\rho_{S^\*}}{(1+\Omega \|r\|_2)^2}\left((1+\Omega \|r\|_2)\|u\|_2^2 - \Omega\|r\|_2\|u\|_2^2\right)\\
  &= \frac{2a\rho_{S^\*}}{1+\Omega \|r\|_2}\|u\|_2^2.
\end{align*}
As a consequence, for any $r, \bar{r} \in \R^n$,
\begin{align*}
  \langle \bar{r}-r, \nabla V(\bar{r})- \nabla V(r) \rangle &= \left\langle \bar{r}-r, \int_0^1 \frac{\dd}{\dd t} \nabla V(t\bar{r} + (1-t)r)\dd t\right\rangle\\
  &= \int_0^1 \left\langle \bar{r}-r,  \nabla^2 V(t\bar{r} + (1-t)r)(\bar{r}-r)\right\rangle\dd t\\
  &\geqslant 2a\rho_{S^\*}\int_0^1 \frac{\|\bar{r}-r\|^2_2}{1+\Omega \|t\bar{r} + (1-t)r\|_2}\dd t\\
  &\geqslant 2a\rho_{S^\*} \frac{\|\bar{r}-r\|^2_2}{1+\Omega (\|\bar{r}\|_2 \vee \|r\|_2)}.
\end{align*}

  \emph{Step~2. Global estimate.} Using only the nonnegativity of $\langle \RRdb(t)-\RRd(t), -\nabla V(\RRdb(t)) + \nabla V(\RRd(t))\rangle$, we get, for any $t \geqslant 0$,
  \begin{align*}
    \frac{\dd}{\dd t} \|\RRdb(t)-\RRd(t)\|^2_2 &= 2 \left\langle \RRdb(t)-\RRd(t), -\nabla V(\RRdb(t)) + 2b \ee^{-\mumin t} \frac{\RRd(t)}{\|\RRd(t)\|_2} + \nabla V(\RRd(t))\right\rangle\\
    &\leqslant 2b \ee^{-\mumin t}\|\RRdb(t)-\RRd(t)\|_2,
  \end{align*}
  and therefore the global estimate on $\|\RRdb(t)-\RRd(t)\|_2$ follows from Lemma~\ref{lem:gronwall-sqrt}.
  
  \emph{Step~3. Long time convergence.} We shall make use of this global estimate to prove the long time convergence of $\|\RRdb(t)-\RRd(t)\|_2$ to $0$. We first note that, for any $r, \bar{r} \in \R^n$, if $\|\bar{r}-r\|_2 \leqslant b/\mumin$, then $\|\bar{r}\|_2 \vee \|r\|_2 \leqslant \|\bar{r}\|_2 + b/\mumin$ and thus
\begin{equation*}
  \langle \bar{r}-r, \nabla V(\bar{r})- \nabla V(r) \rangle \geqslant 2a\rho_{S^\*} \frac{\|\bar{r}-r\|^2_2}{1+\Omega (\|\bar{r}\|_2 + b/\mumin)}.
\end{equation*}
Let us now fix $M > 1$ large enough for the inequality
\begin{equation*}
  \mu' := 2a\rho_{S^\*}/(1+\Omega (M + b/\mumin)) < \mumin
\end{equation*}
to hold, and define
\begin{align*}
  \tau^-_0 &:= \inf\{t \geqslant 0: \|\RRdb(t)\|_2 \leqslant M-1\},\\
  \tau^+_\ell &:= \inf\{t \geqslant \tau^-_\ell: \|\RRdb(t)\|_2 \geqslant M\}, \qquad \ell \geqslant 0,\\
  \tau^-_\ell &:= \inf\{t \geqslant \tau^+_{\ell-1}: \|\RRdb(t)\|_2 \leqslant M-1\}, \qquad \ell \geqslant 1.
\end{align*}
Since $\|\RRdb(t)\|_2^2$ satisfies a time-homogeneous SDE and is ergodic, the sequences of positive random variables $(\tau^+_\ell-\tau^-_\ell)_{\ell \geqslant 0}$ and $(\tau^-_{\ell+1}-\tau^+_\ell)_{\ell \geqslant 0}$ are well-defined, and by the strong Markov property they are iid and therefore $\tau^-_\ell, \tau^+_\ell \to +\infty$ with $\ell$. Besides, on account on the previous discussion, we have
\begin{equation*}
  \frac{\dd}{\dd t}  \|\RRdb(t)-\RRd(t)\|^2_2 \leqslant -2\mu' \|\RRdb(t)-\RRd(t)\|^2_2 + 2b \ee^{-\mumin t}\|\RRdb(t)-\RRd(t)\|_2, \qquad \text{on $[\tau^-_\ell, \tau^+_\ell]$,}
\end{equation*}
and
\begin{equation*}
  \frac{\dd}{\dd t} \|\RRdb(t)-\RRd(t)\|^2_2 \leqslant 2b \ee^{-\mumin t}\|\RRdb(t)-\RRd(t)\|_2, \qquad \text{on $[\tau^+_\ell, \tau^-_{\ell+1}]$.}
\end{equation*}
We deduce from Lemma~\ref{lem:gronwall-sqrt} that
\begin{equation*}
  \|\RRdb(\tau^-_0)-\RRd(\tau^-_0)\|_2 \leqslant \frac{b}{\mumin}\left(1-\e^{-\mumin \tau^-_0}\right),
  \end{equation*}
  and for any $t \in [\tau^-_\ell, \tau^+_\ell]$,
  \begin{equation*}
  \|\RRdb(t)-\RRd(t)\|_2 \leqslant \e^{-\mu'(t-\tau^-_\ell)}\|\RRdb(\tau^-_\ell)-\RRd(\tau^-_\ell)\|_2 + \frac{b\e^{-\mu' t}}{\mumin-\mu'}\left(\e^{-(\mumin-\mu') \tau^-_\ell} - \e^{-(\mumin-\mu') t}\right),\end{equation*}
  and for any $t \in [\tau^+_\ell, \tau^-_{\ell+1}]$,
  \begin{equation*}
  \|\RRdb(t)-\RRd(t)\|_2 \leqslant \|\RRdb(\tau^+_\ell)-\RRd(\tau^+_\ell)\|_2 + \frac{b}{\mumin}\left(\e^{-\mumin \tau^+_\ell} - \e^{-\mumin t}\right).
\end{equation*}
As a consequence, the sequence $Q_\ell := \|\RRdb(\tau^-_\ell)-\RRd(\tau^-_\ell)\|_2$ satisfies the recursive inequation $ Q_{\ell+1} \leqslant \alpha_\ell Q_\ell + \beta_\ell$, where
\begin{equation*}
  \alpha_\ell := \e^{-\mu'(\tau^+_\ell-\tau^-_\ell)}, \quad \beta_\ell := \frac{b\e^{-\mu' \tau^+_\ell}}{\mumin-\mu'}\left(\e^{-(\mumin-\mu') \tau^-_\ell} - \e^{-(\mumin-\mu') \tau^+_\ell}\right) + \frac{b}{\mumin}\left(\e^{-\mumin \tau^+_\ell} - \e^{-\mumin \tau^-_{\ell+1}}\right),
\end{equation*}
and moreover we have the intermediate (rough) control
\begin{equation*}
  \sup_{\tau^-_\ell \leqslant t \leqslant \tau^-_{\ell+1}} \|\RRdb(t)-\RRd(t)\|_2 \leqslant Q_\ell + \frac{b}{\mumin-\mu'}\e^{-(\mumin-\mu') \tau^-_\ell} + \frac{b}{\mumin}\e^{-\mumin \tau^+_\ell},
\end{equation*}
so that, since $\tau^-_\ell, \tau^+_\ell \to +\infty$, to show that $\|\RRdb(t)-\RRd(t)\|_2 \to 0$ is suffices to show that $Q_\ell \to 0$. The recursive inequation yields, for any $\ell \geqslant 1$,
\begin{equation*}
  Q_\ell \leqslant \left(\prod_{m=0}^{\ell-1} \alpha_m\right) Q_0 + \sum_{m=0}^{\ell-1} \left(\prod_{k=m+1}^{\ell-1} \alpha_k\right) \beta_m.
\end{equation*}
On the one hand,
\begin{equation*}
  \prod_{m=0}^{\ell-1} \alpha_m = \exp\left(-\mu' \sum_{m=0}^{\ell-1} (\tau^+_m-\tau^-_m)\right),
\end{equation*}
and since the sequence $(\tau^+_m-\tau^-_m)_{m \geqslant 0}$ is iid with $\Pr(\tau^+_0-\tau^-_0 > 0)=1$, we have $\sum_{m=0}^{\ell-1} (\tau^+_m-\tau^-_m) \to +\infty$, almost surely. On the other hand, we first note that
\begin{align*}
  \beta_m &= \frac{b}{\mumin-\mu'}\left(\e^{-\mu'(\tau^+_m-\tau^-_m)}\e^{-\mumin \tau^-_m} - \e^{-\mumin \tau^+_m}\right) + \frac{b}{\mumin}\left(\e^{-\mumin \tau^+_m} - \e^{-\mumin \tau^-_{m+1}}\right)\\
  &\leqslant \frac{b}{\mumin-\mu'}\left(\e^{-\mumin \tau^-_m} - \e^{-\mumin \tau^-_{m+1}}\right),
\end{align*}
so that
\begin{align*}
  \sum_{m=0}^{\ell-1} \left(\prod_{k=m+1}^{\ell-1} \alpha_k\right) \beta_m &\leqslant \frac{b}{\mumin-\mu'}\sum_{m=0}^{\ell-1} \left(\prod_{k=m+1}^{\ell-1} \alpha_k\right) \left(\e^{-\mumin \tau^-_m} - \e^{-\mumin \tau^-_{m+1}}\right)\\
  &= \frac{b}{\mumin-\mu'}\sum_{m=0}^{\ell-1} \e^{-\mumin \tau^-_m} (1-\alpha_m) \left(\prod_{k=m+1}^{\ell-1} \alpha_k\right) - \e^{-\mumin \tau^-_\ell}.
\end{align*}
Since $\tau^-_\ell \to +\infty$ when $\ell \to +\infty$, the remainder $\e^{-\mumin \tau^-_\ell}$ vanishes when $\ell \to +\infty$. Besides, for any $\varepsilon > 0$, there exists $m_0$ such that for all $m \geqslant m_0$, $\e^{-\mumin \tau^-_m} \leqslant \varepsilon$. We may then write, for $\ell-1 \geqslant m_0$,
\begin{equation*}
  \sum_{m=0}^{\ell-1} \e^{-\mumin \tau^-_m} (1-\alpha_m) \left(\prod_{k=m+1}^{\ell-1} \alpha_k\right) \leqslant \sum_{m=0}^{m_0-1} \e^{-\mumin \tau^-_m} (1-\alpha_m) \left(\prod_{k=m+1}^{\ell-1} \alpha_k\right) + \varepsilon\sum_{m=m_0}^{\ell-1} (1-\alpha_m) \left(\prod_{k=m+1}^{\ell-1} \alpha_k\right).
\end{equation*}
Each one of the $m_0$ terms of the sum in the first term of the right-hand side goes to $0$ when $\ell \to +\infty$, while the second term is telescopic and rewrites
\begin{equation*}
  \varepsilon\sum_{m=m_0}^{\ell-1} (1-\alpha_m) \left(\prod_{k=m+1}^{\ell-1} \alpha_k\right) = \varepsilon\sum_{m=m_0}^{\ell-1} \left(\prod_{k=m+1}^{\ell-1} \alpha_k-\prod_{k=m}^{\ell-1} \alpha_k\right) \leqslant \varepsilon.
\end{equation*}
Therefore, for any $\varepsilon>0$, we get $\limsup_{\ell \to +\infty} Q_\ell \leqslant \varepsilon$, which proves that $Q_\ell \to 0$, almost surely, and finally $\|\RRdb(t)-\RRd(t)\|_2 \to 0$ as well.
\end{proof}

\begin{lemma}[A Gronwall-type estimate]\label{lem:gronwall-sqrt}
  Let $0 \leqslant \mu' < \mu$ and $b \geqslant 0$. Assume that $u(t) \geqslant 0$ satisfies
  \begin{equation*}
    u'(t) \leqslant -2\mu' u(t) + 2b \ee^{-\mu t} \sqrt{u(t)}
  \end{equation*}
  on some interval $[t_1,t_2]$. Then
  \begin{equation*}
    \forall t \in [t_1,t_2], \qquad \sqrt{u(t)} \leqslant \e^{-\mu'(t-t_1)}\sqrt{ u(t_1)} + \frac{b\ee^{-\mu' t}}{\mu-\mu'} \left(\ee^{-(\mu-\mu')t_1}-\ee^{-(\mu-\mu')t}\right).
  \end{equation*}
\end{lemma}
\begin{proof}
  We first set $\hat{u}(t) = \e^{2\mu' t} u(t)$, so that
  \begin{equation*}
    \hat{u}'(t) \leqslant 2b \ee^{(2\mu'-\mu) t} \sqrt{u(t)} = 2b \ee^{(\mu'-\mu) t} \sqrt{\hat{u}(t)}.
  \end{equation*}
  For any $\varepsilon>0$, we therefore deduce that the function $\hat{u}_\varepsilon$ defined by $\hat{u}_\varepsilon(t) = \hat{u}(t) + \varepsilon$ satisfies $\hat{u}_\varepsilon(t) > 0$ and, for any $t \in [t_1,t_2]$,
  \begin{equation*}
    \hat{u}_\varepsilon'(t) = \hat{u}'(t) \leqslant 2b \ee^{(\mu'-\mu) t} \sqrt{\hat{u}(t)} \leqslant 2b \ee^{(\mu'-\mu) t} \sqrt{\hat{u}_\varepsilon(t)}.
  \end{equation*}
  Therefore
  \begin{equation*}
    \sqrt{\hat{u}_\varepsilon(t)}-\sqrt{\hat{u}_\varepsilon(t_1)} = \int_{t_1}^t \frac{\hat{u}'_\varepsilon(s)}{2\sqrt{\hat{u}_\varepsilon(s)}}\dd s \leqslant b\int_{t_1}^t \ee^{(\mu'-\mu) s}\dd s.
  \end{equation*}
  We deduce that 
  \begin{equation*}
    \sqrt{\e^{2\mu' t} u(t) + \varepsilon} \leqslant \sqrt{\e^{2\mu' t_1} u(t_1) + \varepsilon} + b\int_{t_1}^t \ee^{(\mu'-\mu) s}\dd s,
  \end{equation*}
  in which we take the $\varepsilon \to 0$ limit to rewrite
  \begin{equation*}
    \sqrt{u(t)} \leqslant \e^{-\mu'(t-t_1)}\sqrt{ u(t_1)} + b\e^{-\mu' t}\int_{t_1}^t \ee^{(\mu'-\mu) s}\dd s,
  \end{equation*}
  and obtain the claimed estimate.
\end{proof}

We are now ready to complete the proof of Theorem~\ref{thm:standard_noise} (see next subsection for the concentration property). To proceed, we first note that it is easily checked that
\begin{equation*}
  \Zd := \int_{\R^n} \exp\left(-\frac{V(r)}{2\delta}\right)\dd r < +\infty,
\end{equation*}
which allows to define the probability measure
\begin{equation*}
  \Mud_\infty(\dd r) := \frac{1}{\Zd}\exp\left(-\frac{V(r)}{2\Rdelta}\right)\dd r
\end{equation*}
on $\R^n$. By standard arguments, the process $(\RRdb(t))_{t \geqslant 0}$ is ergodic with respect to the probability measure $\Mud_\infty$. In particular, $\RRdb(t)$ converges in distribution towards $\RRd_\infty \sim \Mud_\infty$, which combined with Proposition~\ref{prop:RRdRRdb} entails that $\RRd(t)$ converges in distribution to $\RRd_\infty$, and therefore $\Xid(t)$ converges in distribution to $\Xid_\infty = \|\RRd_\infty\|_2^2$. 

Outside the support, it remains to show that
\begin{equation*}
  \limsup_{t \to +\infty} \|\RX_{{S^\*}^c}^\top\|\sqrt{\Xid(t)}\iind{{S^\*}^c} - \muL_{{S^\*}^c} t = -\infty, \qquad \text{almost surely.}
\end{equation*}
We start by noting that, for any $t \geqslant 0$,
  \begin{equation*}
    \frac{\RRdb(t)}{t} = \frac{r(0)}{t} - \frac{1}{t}\int_0^t \nabla V(\RRdb(s))\dd s + 2\sqrt{\Rdelta}\frac{B(t)}{t}.
  \end{equation*}
  Since $V$ is even, by the ergodic theorem,
  \begin{equation*}
    \lim_{t \to +\infty} \frac{1}{t}\int_0^t \nabla V(\RRdb(s))\dd s = \int_{\R^n} \nabla V(r) \Mud_\infty(\dd r) = 0, \qquad \text{almost surely,}
  \end{equation*}  
  while by the law of the iterated logarithm, 
  \begin{equation*}
    \lim_{t \to +\infty} \frac{r(0)}{t} + 2\sqrt{\Rdelta}\frac{B(t)}{t} = 0, \qquad \text{almost surely.}
  \end{equation*}  
  Therefore,
\begin{equation*}
  \lim_{t \to +\infty} \frac{\RRdb(t)}{t} = 0, \qquad \text{almost surely,}
\end{equation*}
and by Proposition~\ref{prop:RRdRRdb} we deduce that
\begin{equation*}
  \lim_{t \to +\infty} \frac{\RRd(t)}{t} = 0, \qquad \text{almost surely.}
\end{equation*}
Since the processes $(\|\RRd(t)\|)_{t \geqslant 0}$ and $(\sqrt{\Xid(t)})_{t \geqslant 0}$ have the same law, we deduce that
\begin{equation*}
  \lim_{t \to +\infty} \frac{\sqrt{\Xid(t)}}{t} = 0, \qquad \text{almost surely,}
\end{equation*}
which proves the claim.

\subsubsection{Poincaré inequality and concentration properties} 

We first show, in Lemma~\ref{lem:poinca}, that the probability measure $\Mud_\infty$ satisfies a Poincaré inequality with a constant $\kappa^{\Rdelta}$ on which we provide an explicit bound. We recall that it means that, for any smooth function $f$ on $\R^n$,
\begin{equation*}
  \int_{\R^n} \left(f(r)-\int_{\R^n} f(r')\Mud_\infty(\dd r')\right)^2\Mud_\infty(\dd r) \leqslant \kappa^{\Rdelta} \int_{\R^n} \|\nabla f(r)\|^2\Mud_\infty(\dd r).
\end{equation*} 
From this result, the concentration inequality~\eqref{eq:concentration_xi} stated in Theorem~\ref{thm:standard_noise} follows from Eq.~(4.4.6), p.~192 in \cite{bakry2014analysis}. Last, the fact that $\kappa^{\Rdelta} = \mathcal{O}(\Rdelta)$ in the $\Rdelta \to 0$ regime follows from a basic application of the Laplace method, which is detailed in Remark~\ref{rk:laplace}.
\begin{lemma}[Poincare constant]\label{lem:poinca}
  The probability measure $\mu^\delta_\infty$ satisfies a Poincaré inequality with constant $\kappa^{\Rdelta} \leqslant 13 (\sigma^{\Rdelta})^2/n$, where
  \begin{equation*}
    \sigma^{\Rdelta} := \sqrt{\int_{\R^n} \|r\|^2_2 \Mud_\infty(\dd r)} = \sqrt{\mathbb{E}[\|\RRd_\infty\|_2^2]}.
  \end{equation*}
\end{lemma}
\begin{proof}
  Let us denote by $\Mud_\infty(r)$ the density of the probability measure $\Mud_\infty(\dd r)$ with respect to the Lebesgue measure on $\R^n$. It is an immediate computation to show that the probability density $\tilde{\mu}^{\Rdelta}_\infty$ defined by
  \begin{equation*}
    \tilde{\mu}^{\Rdelta}_\infty(\tilde{r}) := \left(\frac{\sigma^{\Rdelta}}{\sqrt{n}}\right)^n \Mud_\infty\left(\frac{\sigma^{\Rdelta}}{\sqrt{n}}\tilde{r}\right)
  \end{equation*}
  satisifes
  \begin{equation*}
    \int_{\R^n} \|\tilde{r}\|_2^2 \tilde{\mu}^{\Rdelta}_\infty(\tilde{r})\dd \tilde{r} = n,
  \end{equation*}
  and that this density writes as a log-concave function of $\|\tilde{r}\|_2$. Therefore, by Theorem~1 in~\cite{bobkov2003spectral}, it satisfies a Poincaré inequality with constant $\tilde{\kappa}^{\Rdelta} \leq 13$, and the final statement follows from an elementary rescaling argument for Poincaré inequalities.
\end{proof}

\begin{remark}[Estimate of the Poincare constant]\label{rk:laplace}
Finally let us give some estimate on the Poincaré constant when $\delta$ is small. To give an equivalent of $(\sigma^{\Rdelta})^2 = \int_{\R^n} \|r\|^2_2 \Mud_\infty(\dd r)$ with the Laplace method, we approximate $V$ by a quadratic near its minimum in $0_{\R^n}$, $V(r) \simeq V(0) + \frac{1}{2} \langle r, \nabla^2 V(0) r\rangle = a\rho_{S^\*} \|r\|^2_2$ by the equation~\eqref{eq:potential}. Therefore by the Laplace method, in the small $\Rdelta$ regime, the distribution with density $\mu_\infty^\Rdelta$ should approximately behave as a centered Gaussian random variable, with covariance matrix $\delta/(a \rho_{S^\*}) I_n$. In particular, this gives an upperbound of the Poincaré constant: 
\begin{equation*}
\kappa^{\Rdelta} \leqslant \frac{13(\sigma^{\Rdelta})^2}{n} \simeq \frac{13 \delta}{a\rho_{S^\*}} = \mathcal{O}(\delta).
\end{equation*} 
Finally remark that in the small $\delta$ limit, this estimate is tight up to the numerical constant as a lower bound with numerical constant $1$ is also given in \cite{bobkov2003spectral}.
\end{remark}

\end{document}